\documentclass[aip,reprint,footinbib,reqno]{amsart}
\makeatletter
\def\subsection{\@startsection{subsection}{2}%
  \z@{.5\linespacing\@plus.7\linespacing}{.5\linespacing}%
  {\normalfont\bfseries}}
\makeatother
\usepackage[foot]{amsaddr}
\allowdisplaybreaks

\usepackage{graphicx,graphics,caption, subcaption,color}
\graphicspath{{fig/}} % 指定图片文件夹路径
\usepackage{enumitem}

\usepackage{amsmath,amssymb,amsfonts,bm,empheq}
\usepackage{cases}
\usepackage{graphicx}
\graphicspath{{fig/}}

\usepackage{tabularx}  % 支持自动调整列宽
\usepackage{booktabs}  % 用于创建漂亮的水平线
\usepackage{multirow}   % 用于合并多行
\usepackage{threeparttable}
\usepackage{xcolor} 
\usepackage{soul} 

\usepackage[ruled,linesnumbered]{algorithm2e}
\usepackage{lineno}

\usepackage{epstopdf}

\usepackage{hyperref}
 \hypersetup{
    colorlinks=true,       % false: boxed links; true: colored links 
    linkcolor=red,          % color of internal links (change box color with %linkbordercolor)
     citecolor=blue,        % color of links to bibliography
     urlcolor=cyan           % color of external links
 }

\usepackage{setspace}
\usepackage{diagbox} % 加载宏包

\usepackage[colorinlistoftodos,prependcaption,textsize=tiny]{todonotes}

\usepackage{soul}
\soulregister\ref7
\soulregister\cite7
\soulregister\eqref7
\usepackage{changes}
\usepackage{cancel}

\definechangesauthor[name={Xiang ZHOU}, color=red]{x}
\setdeletedmarkup{\textcolor{black}{\sout{#1}}}

\newcommand{\inner}[2]{\langle#1,#2\rangle}

\newlength\mylen

\newtheorem{theorem}{Theorem}[section]

\newtheorem{proposition}[theorem]{Proposition}
\newtheorem{remark}{Remark}

 \newcommand{\abs}[1]{\left\vert #1\right\vert}

\renewcommand{\d}{\ensuremath{\mathrm{d}}}
\newcommand{\dt}{ \ensuremath{\mathrm{d} t } }
\newcommand{\dx}{ \ensuremath{\mathrm{d} x} }
\newcommand{\dy}{ \ensuremath{\mathrm{d} y } }
\newcommand{\dz}{ \ensuremath{\mathrm{d} z } }
\newcommand{\dW}{ \ensuremath{\mathrm{d} W } }

\def\H_1_u{H_u^{-1}}

\usepackage[a4paper,margin=1.2in]{geometry}
\usepackage{physics}

\newtheorem{assumption}{Assumption}

\begin{document}
 
\title {  Weak Generative Sampler to Efficiently  Sample Invariant Distribution of Stochastic Differential Equation}

 \author{ Zhiqiang Cai$^\dag$, Yu Cao$^\ddag$, Yuanfei Huang$^\S$, Xiang Zhou$^\star$$^\P$}

 \address{$^\star$Corresponding author.}
 \address{$^\dag$Department of Data Science, City University of Hong Kong, Kowloon, Hong Kong SAR, zqcai3-c@my.cityu.edu.hk}
 \address{$^\ddag$Institute of Natural Sciences and School of Mathematical Sciences, Shanghai Jiao Tong University, Shanghai 200240, China, yucao@sjtu.edu.cn. YC is supported by NSFC grant No. 12401573 and is sponsored by Shanghai Pujiang Program (23PJ1404600).}
 \address{$^\S$Department of Data Science, City University of Hong Kong, Kowloon, Hong Kong SAR, yhuan26@cityu.edu.hk}
 \address{$^\P$ Department of Mathematics and Department of Data Science, City University of Hong Kong, Kowloon, Hong Kong SAR, xizhou@cityu.edu.hk.
 XZ acknowledges the support from 
 Hong Kong General Research Funds  (11308121, 11318522,   11308323),  and  the NSFC/RGC Joint Research Scheme [RGC Project No. N-CityU102/20 and NSFC Project No.
12061160462]. 
}
\address{{The authors thank the anonymous referees for their helpful comments that improved the
quality of the manuscript.}}

% \date{\today}

\begingroup
\def\uppercasenonmath#1{} % this disables uppercasing title
\let\MakeUppercase\relax % this disables uppercasing authors
\maketitle
\endgroup

\begin{abstract}
Sampling invariant distributions from an It\^o diffusion process presents a significant challenge in stochastic simulation. Traditional numerical solvers for stochastic differential equations require both a fine step size and a lengthy simulation period, resulting in biased and correlated samples. The current deep learning-based method solves the stationary Fokker--Planck equation to determine the invariant probability density function in the form of deep neural networks, but they generally do not directly address the problem of sampling from the computed density function.  In this work, we introduce a framework that employs a weak generative sampler (WGS) to directly generate independent and identically distributed (iid) samples induced by a transformation map derived from the stationary Fokker--Planck equation. Our proposed loss function is based on the weak form of the Fokker--Planck equation, integrating normalizing flows to characterize the invariant distribution and facilitate sample generation from a base distribution. Our randomized test function circumvents the need for min-max optimization in the traditional weak formulation. Our method necessitates neither the computationally intensive calculation of the Jacobian determinant nor the invertibility of the transformation map.  A crucial component of our framework is the adaptively chosen family of test functions in the form of Gaussian kernel functions with centers related to the generated data samples. Experimental results on several benchmark examples demonstrate the effectiveness and scalability of our method, which offers both low computational costs and excellent capability in exploring multiple metastable states.

\end{abstract}

%\tableofcontents


 {\noindent {\bf MSC2020}: 65N75, 65C20, 68T07

\section{Introduction}
Stochastic differential equations (SDEs) have been widely employed to model the evolution of dynamical systems under uncertainty. They arise in many disciplines such as physics, chemistry, biology, and finance. For many realistic models, the system will reach a dynamical equilibrium in the long run, that is, the probability distribution of the system will reach an invariant measure.
Computing and sampling this invariant distribution is a long-standing computational problem with applications across diverse disciplines: for instance, sampling from the invariant measure facilitates more efficient exploration of phase space, thereby enhancing our comprehension of rare events \cite{cousins2016reduced, mohamad2016probabilistic,zhang2022koopman} and aiding in the estimation of physical quantities for certain distributions \cite{cui2024multilevel, mathias_rousset_free_2010} and studying the free energy \cite{darve2001calculating}, the Bayesian data assimilation \cite{reich2015probabilistic}  as well as studying structural biology matter \cite{wuttke2014temperature}.

In this work, we consider the following SDE on $\mathbb{R}^d$:
\begin{equation}\label{eqn:sde}
\d X_t=b(X_t)\dt+\sqrt{2}\sigma(X_t) \dW_t,
\end{equation}
where the vector-valued function $b: \mathbb{R}^d\rightarrow \mathbb{R}^d$ is continuous, and $\sigma:\mathbb{R}^{d} \to \mathbb{R}^{d\times w}$  is a matrix-valued function, and $W_t$ is a $ w$-dimensional standard Brownian motion. 

The probability density of the SDE at time $t$, denoted as $p_t$, is known to evolve according to the Fokker--Planck equation $\partial_t p_t = \mathcal{L} p_t$, where the differential operator 
\begin{equation}
\label{eqn::fp}
   \mathcal{L}p := \nabla\cdot\big(-bp\big) + \nabla^2: \big( Dp\big) ,
\end{equation}
and where  $\nabla^2:(Dp)=\sum_{ij}\partial_{ij}(D_{ij}p)$ and the diffusion matrix $D=\sigma\sigma^\top=(D_{ij})$ satisfies the uniform ellipticity that $\lambda I \le D(x) \le \lambda^{-1} I$ for all $x$ with a positive constant $\lambda$. The invariant distribution $p$ is the long time limit of the distribution of $X_t$, 
$p:=\underset{t\to +\infty}{\lim}p_t$.
If the invariant distribution $p$  exists uniquely regardless of the initial distribution $p_0$, this SDE  \eqref{eqn:sde}  is called ergodic. 
 Under the mild assumption of the drift $b(x)$ and the uniform elliptic diffusion $\sigma(x)$ (see Assumption \ref{assum:SDE} in Section \ref{s21} below), the ergodicity holds.  For example, if there is a compact domain $A \subset \mathbb{R}^d$ such that all the flows associated with the vector field $b$ are attracted \footnote{This means that for any $x_0 \in A$, the solution $x_t$ of the ODE $\dot{x}_t=b(x_t)$ is always in $A$ for any $t>0$.} to $A$ and 
$b(x)\cdot n(x) <0$ for all $x \in \partial A$, where $n(x)$ is the exterior normal of the boundary of $A$, then the SDE  \eqref{eqn:sde} 
 is ergodic.

 Estimating the invariant distribution $p$ can be achieved by finding the zero eigen-state of the operator $\mathcal{L}$:
\begin{equation}\label{eqn:sfpe}
   \mathcal{L}p = 0.
\end{equation}
which is also known as the \emph{stationary Fokker--Planck equation} (SFPE).

 Traditional techniques in numerical PDEs such as the finite difference method or finite element method \cite{filbet2002numerical, kumar2006solution}, can be utilized to solve SFPE.
However, these methods encounter challenges due to the \emph{curse of dimensionality}.
Moreover, directly estimating the invariant distribution does not inherently provide a scheme for efficiently generating samples from this distribution, which is also a core task in many applications.

To overcome the dimensionality limitation and to achieve the goal of sample generation, Monte Carlo-based methods have been extensively studied in the literature. 
A typical approach is to adopt
numerical schemes \cite{darve2009computing,ermak_numerical_1980,gilsing2007sdelab, hutzenthaler2015numerical} to evolve the SDE for a sufficiently long time to generate the samples of the invariant distribution.  
There is an important class of reversible process where the invariant distribution has the known expression up to a constant:
when the drift term of the SDE has the gradient form $b = -\nabla U$ for some given potential function $U$ and the diffusion coefficient is constant, the associated SDE is known as the overdamped Langevin dynamics \cite{cheng_underdamped_2018, pavliotis_stochastic_2014};
the invariant measure is then simply the Boltzmann distribution $\propto \exp(-U(x)/k_B T)$ where $k_B$ is the Boltzmann constant and $T$ is the thermodynamic temperature.
Under certain assumptions of confining potential,   the density function $p_t$ will converge to its unique equilibrium $p$ exponentially fast as $t$ goes to infinity \cite{pavliotis_stochastic_2014}.
This  fast mixing property is an important ingredient for the efficiency of
many Langevin-based sampling algorithms by adopting various numerical discretization schemes \cite{brunger1984stochastic, cheng_underdamped_2018, dalalyan_sampling_2020, ermak_numerical_1980, gronbech2013simple, leimkuhler2013rational, shen_randomized_2019}.
However, when the stochastic system with non-convex potential exhibits meta-stability,
it takes an extremely large time for $X_t$ to converge to the equilibrium  under low temperature \cite{pavliotis_stochastic_2014}.
This challenge has attracted much attention, which has been addressed via e.g., parallel tempering method \cite{lu_methodological_2019, yu_multiscale_2016}, annealing-based methods \cite{neal_annealed_2001}. As a remark, different from the Langevin-type dynamics, our approach below is broad and we study the general form of the SDE or the Fokker--Planck equation, provided that the invariant distribution exists, without requiring that the drift term $b$ should be in the gradient form.

 With the unprecedented success of deep neural networks in powerful expressiveness,
many machine learning techniques have been developed in the past few years to parameterize and solve high-dimensional partial differential equations \cite{bruna2024neural, weinan2021algorithms, han2018solving}.
Notably, the \emph{deep Ritz method} \cite{e_deep_2018} is an early pioneering work in this area by exploiting the variational formulation of Poisson equations. 
The \emph{physics-informed neural network} (PINN) \cite{raissi2019physics} proposed to directly incorporate the structure of PDE into the loss function.
For the particular problems to study in this work, deep learning like PINN  is currently the backbone for many methods to tackle the solution of SFPE \cite{gu2023stationary, krishnapriyan2021characterizing, sun2022data, xu2020solving, zhai2022deep}. 
Another important method, called  \emph{weak adversarial network} (WAN), was proposed to replace the $L^2$ loss in PINN via a min-max problem, whose flexibility is an important feature to develop our methods below. The \emph{weak collocation regression} \cite{lu2024weak} utilized the weak form of the Fokker--Planck equation but focused on the inverse stochastic problems.

As discussed above, these deep learning-based PDE solvers cannot inherently sample the invariant distribution. 
For the machine-learning accelerated sampling methods, generative models play a significant role. 
Many of these models, such as variational autoencoders \cite{kingma2013auto}, generative adversarial networks \cite{goodfellow2014generative}, normalizing flows \cite{dinh2016density,rezende2015variational}, and score-based models \cite{song2020score} aim to learn a mapping that transports a base distribution to a target distribution based on a given dataset of samples. For instance, the samples of the target distribution could be a set of images. 
By sampling from the base distribution, we can readily generate samples from the target distribution using the trained mapping. 
A common situation for applying generative models to sampling is that the potential energy for the Boltzmann distribution is known, and normalizing flow-based methods have been used to facilitate the sampling of the target distribution \cite{gabrie_adaptive_2022,noe_boltzmann_2019}. To save the cost of computing
  Jacobian determinant for full matrix,
\cite{marzouk2016sampling,parno2018transport, zech2022sparse} and 
\cite{tang2022adaptive, KRnet-Wan2022, zeng2023adaptive} adopted the triangular building blocks  for  generative maps. 
We emphasize that our task in this paper is not to generate samples from the given dataset, but from a given stochastic differential equation.

The collective power of deep learning methods like PINN and generative models enable us to simultaneously achieve estimating the density and sample from invariant distribution, while avoiding the curse of dimensionality. Recently,  \cite{tang2022adaptive,zeng2023adaptive} proposed the method called \emph{Adaptive Deep Density Approximation} (ADDA), which is based on PINN to utilize normalizing flow to parameterize the invariant measure, and subsequently employ the PINN loss to train the normalizing flow. However, this method requires the time-consuming computation of the Jacobian determinant because it uses the expression of probability density function $p_\theta$. This issue is worsened in the PINN formalism due to the need to take higher-order derivatives of $p_\theta$ in the differential operator $\mathcal{L}$ in \eqref{eqn::fp}.

\subsection*{Our contributions}

In this paper, we focus on estimating the density distribution $p$ and sampling the invariant measure of the stochastic differential equations \eqref{eqn:sde}. We consider the case where the drift term $b$ and the diffusion matrix $\sigma$ are known, but we do not assume any data either from the SDE simulation or the observation measurement  is available. Our primary goal is to sample the invariant measure of the SDE \eqref{eqn:sde}. It is important to note that the drift term $b$ is {\it not} assumed to have a gradient form, which implies that  the invariant distribution is unlikely to have a simple, closed-form expression.

To address this, we propose a novel method called the \emph{weak generative sampler} (WGS), which samples according to  the invariant measure based on the weak formalism of the stationary Fokker--Planck equation. This allows the loss function to be expressed as an expectation with respect to the invariant measure. We employ normalizing flow  (NF) to parameterize the transport map from the base distribution to the invariant distribution. This approach enables us to approximate the loss function using sample data points generated by the normalizing flow, without the need to compute the pdf and its gradient/Hessian during the training. 

The main contributions of this work are as follows:
\begin{itemize}[leftmargin=1cm]
 
 \item \textbf{Robust and efficient method}
 \begin{enumerate}

    \item \textbf{Jacobian-free generative map for PDE}: 
For our problem of invariant measure associated with the PDE \eqref{eqn:sfpe}, the computation of the loss function does not involve calculating the  Jacobian determinant of the normalizing flow, as no explicit expression of the density function is needed. This can generally accelerate the training process by an order of magnitude.

    \item \textbf{Randomized and adaptive test functions}:  
Our method leverages the weak formulation of SFPE 
but does not rely on min-max optimization.
By randomizing the test function, the algorithm becomes more efficient and opens the door to adaptive training design. This not only reduces computational costs but also helps enhance the robust exploration for the SDE with a multi-modal invariant distribution.
        \end{enumerate}
    \item \textbf{Theoretical interpretation}: We provide a rigorous theoretical analysis to establish the bounds on the squared $L^2$ error between the estimated density function and the true density function; see Theorems \ref{theo:U} and \ref{theo:R}.
    
    \item \textbf{Numerical verification}: We conduct numerical experiments on both low and high-dimensional problems, with or without the presence of meta-stable states. We compare our  WGS with the method in \cite{tang2022adaptive} and demonstrate that the WGS achieves comparable accuracy with a significantly lower computational cost. Furthermore, the  WGS can explore all the meta-stable states in both low temperature and high temperature scenarios.  \end{itemize}

The rest of the paper is organized as follows. In Section \ref{s2}, we will discuss the framework of WGS, including the network structure and, in particular, the construction of the loss function and the test functions. In Section \ref{s3}, we develop the theoretical error analysis for WGS, and in Section \ref{s4}, we present four numerical examples to illustrate the efficacy of WGS. Finally, we conclude the paper in Section \ref{s5}.

\section{Numerical Methods}\label{s2}

First, we explain the weak formalism and the motivations behind our proposed method in Section~\ref{s21}. Next, in Section~\ref{s22}, we develop the weak generative sampler, covering both the theoretical aspects and the empirical loss function that is used in practice. The network structure and algorithms are detailed in Sections~\ref{s23} and \ref{s24}, respectively.

\subsection{Fokker--Planck equations and test functions}
\label{s21}
We assume that the drift term $b$ and the diffusion coefficient $\sigma$ satisfy the following conditions.
\begin{assumption}\label{assum:SDE}
    \begin{enumerate}
        \item There is a positive constant $K_1$ such that
        $2x\cdot b(x) + |\sigma(x)|\leq K_1(|x|^2 + 1).$ for all $x\in\mathbb{R}^d$ and $b$ is locally uniformly continuous.
        \item There is a positive constant $\delta_0$ such that, $\forall R>0$ and $x,z\in\mathbb{R}^d$ satisfying $|x|\leq R$, $|z|\leq R$ and $|x-z|\leq \delta_0$,
        $2(x-z)\cdot(b(x)-b(z)) + |\sigma(x)-\sigma(z)|^2\leq K_R|x-z|^2$ holds,  where $K_R$ is a positive constant.
        \item There is a $\lambda_0>0$ such that $\xi\cdot D(x)\xi\geq \lambda_0|\xi|^2$ for all $x,\xi\in\mathbb{R}^d$. Denote by $\sigma_{\lambda_0}$ the unique symmetric nonnegative definite matrix-valued 
function such that $\sigma^2_{\lambda_0}= D -\lambda_0 I$. And $\sigma_{\lambda_0}$ is locally H\"{o}lder continuous with exponent $\delta_{\lambda_0}>\frac{1}{2}$.
\item There is a positive constant $\alpha$, a compact $C\subset\mathbb{R}^d$, a measurable $f:\mathbb{R}^d\mapsto[1,\infty)$,  and twice continuously differentiable function $V:\mathbb{R}^d\mapsto\mathbb{R}_+$ satisfying $\mathcal{L}V(x)\leq -\alpha f(x) + 1_C(x)$ for all $x\in\mathbb{R}^d$.
    \end{enumerate}
\end{assumption}
Under such an assumption, the existence and uniqueness of the solution of It\^{o} SDE \eqref{eqn:sde} can be promised, and it has a unique invariant distribution $p$; see e.g. \cite[Theorem 2.2, Theorem 2.4, Remark 5.5, Lemma 6.4]{xi2019jump}. The invariant distribution $p$, with the properties $p(x)\geq 0$ and $\int_{\mathbb{R}^d} p(x) \dx =1$,  is governed by the stationary Fokker--Planck equation (SFPE):
\begin{align*}
\mathcal{L}p(x) = 0,\quad  {\forall} x\in \mathbb{R}^d.
\end{align*}
Given any test function $\varphi\in C_c^\infty(\mathbb{R}^d)$, 
  the SFPE then gives
\begin{align}\label{SFPE2}
\inner{\varphi}{\mathcal{L}p} = 0,\quad\forall \varphi\in C_c^\infty(\mathbb{R}^d),
\end{align}
where $\inner{f}{g} := \int_{\mathbb{R}^d} f(x) g(x)\ \dd x$ is the standard $L^2$ inner product $L^2(\mathbb{R}^d)$,
and $C_c^\infty(\mathbb{R}^d)$ is the set of  smooth functions with compact support on 
$\mathbb{R}^d$.

To solve this SFPE, \cite{zang2020weak} proposed the form of the min-max optimization problem involving two neural network functions, one for the solution and another for the test function, by solving
\begin{equation}\label{losswan}
    \min_{p}\max_{\varphi: \|\varphi\|_2=1} \abs\big{\inner{\varphi}{\mathcal{L} p}}^2\quad\text{with}\quad p\geq 0\quad\text{and}\quad\int_{\mathbb{R}^d}p\dx=1,
\end{equation}
where $\|\cdot\|_2$ denotes the $L^2$-norm in $\mathbb{R}^d$.
Note that in theory, the optimal $\varphi$ here is simply $\mathcal{L}p/\|\mathcal{L}p\|_2$ and \eqref{losswan}  essentially   minimizes the $L^2$ loss $\|\mathcal{L}p\|_2$ as in the PINN. More generally, as \cite{RQAIEEE2022} pointed out, 
 if the $L^2$ norm for the test function in \eqref{losswan} is 
replaced by $L^r$ norm, then the loss function \eqref{losswan} becomes $\|\mathcal{L}p\|_s$ with $1/s+1/r=1$  by   H\"{o}lder's inequality. In \cite{zang2020weak}, the test functions $\varphi$ is explicitly optimized within the family of neural network functions, so it can be heuristically seen as the discriminator in the traditional Generative Adversarial Network (GAN)\cite{goodfellow2014generative}.
Note that the loss function \eqref{losswan} still applies
the operator $\mathcal{L}$ on the solution, not 
its adjoint on the test function.

In this formalism, the optimal test function $\varphi$ needs to be approximated in the form of neural network first as a subroutine 
in each iteration of the outer minimization problem.
However, the min-max problem is generally prone to instability and solving the maximization problem typically requires substantial computational resources, which we would like to avoid. Moreover, estimating $\mathcal{L}(p)$ still requires applying the differential operator to the density function $p$.
If one employs a normalizing flow structure to parameterize $p$, calculating the Jacobian determinant—often the most computationally intensive step—becomes unavoidable.

\subsection{Weak Generative Sampler}
\label{s22}

Our method, Weak Generative Sampler, is based on the adjoint operator of the Fokker--Planck operator $\mathcal{L}$ and the representation of the probability by a generative flow map.
Instead of considering \eqref{SFPE2} where    calculation of partial  derivatives of $p$ in  $\mathcal{L}p$ is challenging,  we work with  the  following {\it weak formulation} of the SFPE \cite{FPKbook2015, evans2022partial}
\begin{equation}\label{eqn:wsfpe}
\inner{\mathcal{L}^*\varphi}{p} = 0, \qquad \forall \varphi\in C_c^\infty(\mathbb{R}^d),
\end{equation}
where
\begin{align*}
\mathcal{L}^*\varphi:=b\cdot \nabla \varphi+D:\nabla\nabla \varphi.
\end{align*}
is the infinitesimal generator of the stochastic process in \eqref{eqn:sde} and $D:\nabla\nabla \varphi=\sum_{ij}D_{ij}\partial^2_{ij}\varphi$.  In the above equations, $\mathcal{L}^*$ is the $L^2$ adjoint of $\mathcal{L}$, i.e.,
$\inner{v}{ \mathcal{L}u}=\inner{\mathcal{L}^*v}{u}$ for $u,v \in C_c^2(\mathbb{R}^d)$.

Equation \eqref{eqn:wsfpe} is a system of linear equations (with infinite dimension) for $p$ to satisfy, each associated with a test function $\varphi$.
In parallel to \eqref{losswan}, one can solve the following minimax problem 
\begin{equation} \label{327}
\min_{p}\max_{\varphi: \|\varphi\|_2=1} \abs{ \inner{  \mathcal{L}^* \varphi} {p}  } ^2\quad\text{with}\quad p\geq 0\quad\text{and}\quad\int_{\mathbb{R}^d}p(x) \dx=1.
\end{equation}
However, we shall solve this system using a probabilistic approach by randomizing the test function,  which both circumvents the challenging min-max optimization and facilitates the adaptivity.

More specifically,  we consider the Banach  space $\Omega:= C_c^\infty(\mathbb{R}^d)$ for the test function, and suppose $\mathbb{P}$ is a non-degenerate probability distribution on this Banach space $\Omega$, then solving \eqref{eqn:wsfpe} can be   rewritten as
\begin{equation*}%\label{eqn:wsfpe_i}
   \min_{p} \int_{\Omega} \abs\big{\inner{\mathcal{L}^*\varphi}{p}}^2 \dd \mathbb{P}(\varphi)\quad\text{with}\quad p\geq 0\quad\text{and}\quad\int_{\mathbb{R}^d}p(x) \dx=1.
\end{equation*}
It can be written more intuitively in the expectation form:
\begin{equation}\label{eqn:wsfpe_e}
    \min_p\mathbb{E}_{\varphi\sim\mathbb{P}}\Big[\mathbb{E}_{x\sim p}\big[\mathcal{L}^*\varphi(x)\big]\Big]^2\quad\text{with}\quad p\geq 0\quad\text{and}\quad\int_{\mathbb{R}^d}p(x) \dx=1.
\end{equation}
This family of randomized test functions is the key formalism for our proposed method and we call the objective function in  \eqref{eqn:wsfpe_e} as the {\it randomized  weak loss function} in contrast to 
\eqref{losswan} and \eqref{327}.
This formalism relaxes the worse-case error in \eqref{losswan} or \eqref{327} to the averaged-case error in \eqref{eqn:wsfpe_e},  which is extensively adopted in information-based complexity (see e.g., \cite{traub_information-based_2003}). This relaxation does not affect the global minimizer since $\mathbb{P}$ is non-degenerate ( the exact meaning of non-degeneracy is specified later),  but in fact is helpful to improve training stability.

Note that, in \eqref{eqn:wsfpe_e}, we only need the sample data points of $p$ (without its function expression), which facilitates
the application of generative methods. 
Therefore, we introduce a transport map $G_\theta$, with $\theta$ denoting the generic parameter, to map the base distribution (e.g., Gaussian distribution or uniform distribution) to the target distribution $p$. Then for any $z\sim\rho$, we can obtain the associated samples $x=G_\theta(z)$. Therefore, the minimization problem \eqref{eqn:wsfpe_e} is rewritten as 
 \begin{equation}\label{eqn:eloss}
\min_{G_\theta}\mathbb{E}_{\varphi\sim\mathbb{P}}\bigg[\mathbb{E}_{z\sim \rho}\Big[\mathcal{L}^*\varphi\big(G_\theta(z)\big)\Big]\bigg]^2.
\end{equation}
By the Monte Carlo method, the empirical loss function of the minimization problem \eqref{eqn:eloss} becomes 
\begin{equation}\label{eqn:oldloss}
    \min_{G_\theta}\frac{1}{N_\varphi}\sum_{j=1}^{N_\varphi}\left[\frac{1}{N}\sum_{i=1}^N \mathcal{L}^*\varphi_j\big(G_\theta(z_i)\big)\right]^2,
\end{equation}
where $\varphi_j$ are sampled from the distribution $\mathbb{P} (\varphi)$ and $N_\varphi$ is the number of the test function $\varphi_j$; $z_i$ are sampled from the base distribution $\rho$ and $N$ is the number of sample points $\{z_i\}$.
Unlike  GANs, in our method, the test functions are selected through sampling rather than maximization. We shall show the principled guidance of  adaptively selecting these test functions  in the data-driven approach based on the current $G_\theta$. We name this method the \emph{Weak Generative Sampler} (WGS). 

We now discuss the \emph{non-degenerate} condition for the probability measure $\mathbb{P}$ on the test functions, which should ensure  \eqref{eqn:wsfpe_e}   implies \eqref{eqn:wsfpe}. The equivalent condition of this non-degeneracy condition is  
the full support property that  the probability of any ball, $\mathbb{P}(B(\varphi; r))>0$, for every $\varphi\in \Omega$ and any radius  $r>0$.
Probabilities with full support are typically constructed on locally compact spaces, 
but  the infinite dimensional space $C_0^\infty$ is not locally compact, to ensure the existence of this full-support property, one needs to consider a larger  space, for instance, the  $H^2$ space. Then it can be proven the existence of such probability measures satisfying full support on the $H^2$ space. One can refer to the Appendix \ref{appendix:a} of the construction of such Gaussian probability measures on $H^2$.   In addition, if  $\mathbb{P}$ is absolutely continuous with respect to a non-degenerate probability with full support, then $\mathbb{P}$ has full support too and thus is non-degenerate.
In practice, one can choose the test function in the space $\sum_k c_k\phi_k$ spanned by a set of the orthonormal basis functions  $\{\phi_k\}$ in $\Omega$, such as the Hermite polynomials \cite{da2006introduction}, or eigen-functions of $\mathcal{L}$.
 Then the probability $\mathbb{P}$  supported  on the unit ball  $\{\|\vec{c}\|=1: \vec{c}=(c_k)\}$ is sufficient. Using some reproducing kernel Hilbert space as a dense subset of $\Omega$ 
with kernel mean embedding of distribution is also feasible \cite{MAL060}.

However, if the choice of these test functions are pre-determined and static,  the number of test functions used in computation 
will be huge and the training efficiency will not be satisfactory at all, 
since it is unable to provide the informative guidance during the training of the map $G_\theta$. We propose the following adaptive idea based on the data-informed test functions.

Firstly, 
our numerical scheme here uses the family of  Gaussian kernel functions given by:
\begin{equation}
    \varphi_j(x)=\exp\left(-\frac{1}{2\kappa^2}\|x-\zeta_j\|_2^2\right),\quad j\in\{1,2,\cdots,N_\varphi\},
\end{equation}
where $\zeta_j$ represents the centers and the hyper-parameter $\kappa$ is the scale length determining the width of the kernel. The Gaussian kernel gives the infinitely differentiable functions $\varphi$ that decay at infinity.
Then the distribution  $\mathbb{P}$ on the space of the test function is determined by the distribution of the $N_\varphi$ center points  $\{\zeta_j\}_{j=1}^{N_\varphi}$.
$\kappa$ is a hyper-parameter in the training which can be fixed or adaptively chosen. We consider two extreme cases of $\kappa$.
If $\kappa\to \infty$, then $\varphi_j$ become   constant functions, in the null space of $\mathcal{L}^*$, and 
therefore the weak loss function \eqref{eqn:wsfpe_e} is zero.
If $\kappa\to 0$, then $\inner{\mathcal{L}^*\varphi_j}{p}=\inner{\varphi_j}{\mathcal{L}p}\to \mathcal{L}p(\zeta_j)$
and the loss \eqref{eqn:wsfpe_e} becomes the least-square loss used in the PINN: $ \mathbb{E}_{\zeta}\abs{\mathcal{L}p(\zeta)}^2$.
Many  acceleration  techniques for training the PINN are  based on the intuitive choice of the distribution for the  training sample $\zeta$ \cite{gao2023failure, RQAIEEE2022, Mao2023adaptive, tang2022adaptive, zeng2023adaptive}.
Our  method uses a finite value of $\kappa$ so that the test function $\varphi_j$ reflects the information in a neighbor of size $\kappa$ around, but not strictly limited to,  a point $\zeta_j$;
the rational choice of $\kappa$ will be discussed and validated in later text.  

Secondly, our adaptive choice of the centers $\zeta$ in the test functions $\varphi$ shares the same distribution as  the generated data, which approximates the true solution $p$.
Specifically, $N_\varphi$ is set to be less than $N$ and the collection of $\{\zeta_j\}_{j=1}^{N_\varphi}$ is 
uniformly selected without replacement from  the total  number of $N$  data points $\{x_i = G_\theta(z_i)\}$ used in training the loss \eqref{eqn:oldloss}. To introduce   variability for better exploration, we also add an independent small Gaussian noise to these selected data points, giving 
$$\zeta_j = x_{(j)} + \gamma \mathcal{N}(0, \boldsymbol{I}_d)$$ 
with a small parameter $\gamma>0$, where $x_{(j)}$, $1\le j\le N_\varphi$, are sampled without replacement from the collection $\{x_i, 1\le i\le N\}$.
The size $N_\varphi$ is usually only a small fraction of the total particle numbers $N$. In summary, the choice of the center points for the test function 
is adaptive and informative since it is based on the samples generated from the map $G_\theta$ during the training. 
We illustrate our scheme in the Figure \ref{fig:framework}.
As a side remark, we clarify that this adaptive selection of our test functions $\{\varphi_j\}$ in \eqref{eqn:oldloss} implies the adaptive choice of $\mathbb{P}$ in \eqref{eqn:eloss} during the model training of $\theta$ for the minimum of the loss function. This works because the global minimum of   \eqref{eqn:eloss}   for {\it any} non-generate $\mathbb{P}$ is the desired invariant measure. 

\begin{figure}[ht]
    \centering
    \includegraphics[scale=0.5]{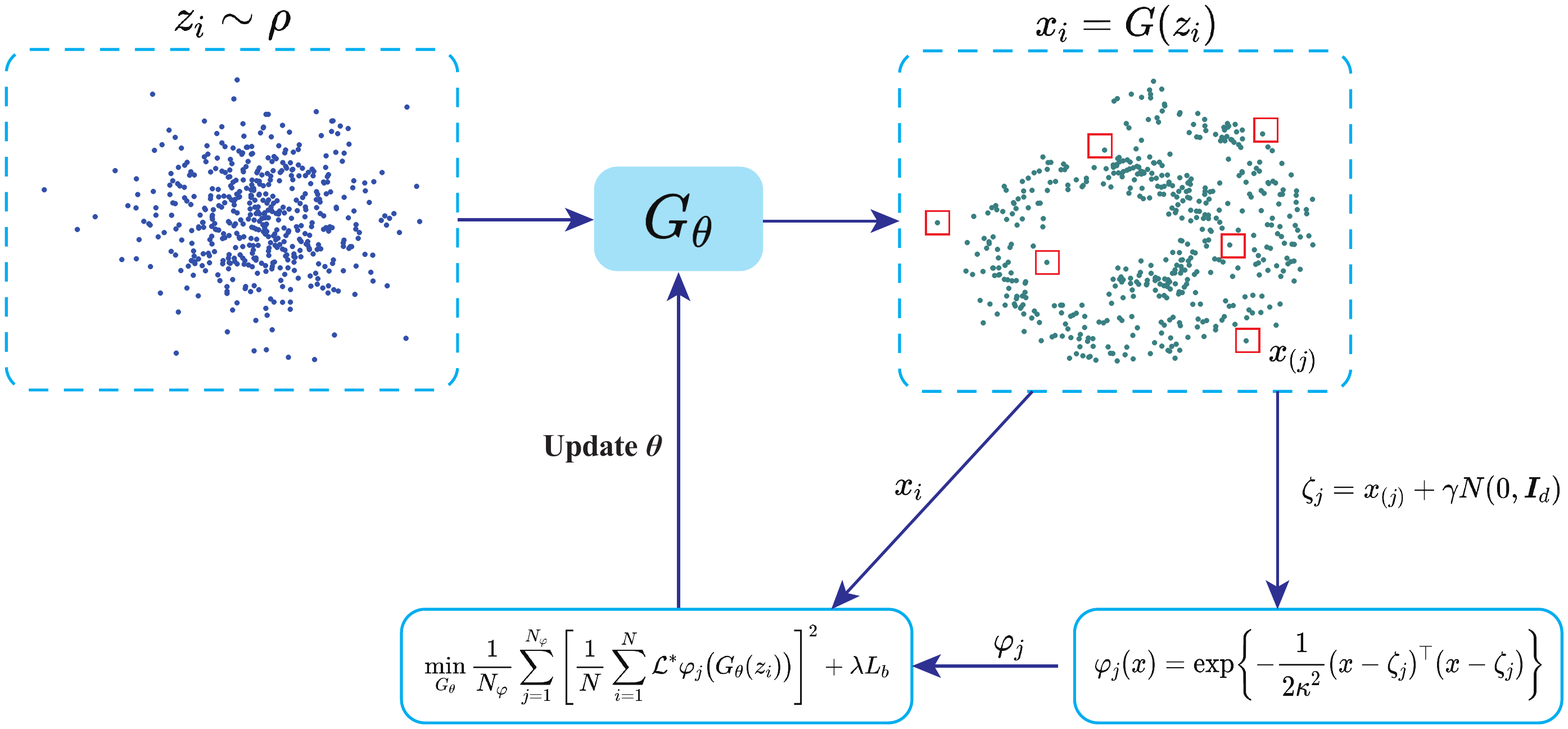}
    \caption{The framework of WGS. In the WGS, data points $\{z_i\}_{i=1}^N$ are sampled from the base distribution $\rho$ and transformed to $\{x_i\}_{i=1}^N$ via the transport map $G_\theta$. Then, $\{x_{(j)}\}_{j=1}^{N_\varphi}$ are selected uniformly from $\{x_i\}_{i=1}^N$, and Gaussian noise is added to obtain $\{\zeta_j\}_{j=1}^{N_\varphi}$. The loss function is derived from data points $\{x_i\}_{i=1}^N$ and test functions $\{\varphi_j\}_{j=1}^{N_\varphi}$. The transport map $G_\theta$ is updated through gradient descent. The term $\lambda L_b$ is a penalty term from the boundary; see  \eqref{eqn:loss} for details.}
    \label{fig:framework}
\end{figure}

\subsection{Normalizing flow and network structure of \texorpdfstring{$G_\theta$}{}}\label{s23}
The transport map $G_\theta$ is a differentiable transformation transporting the base distribution to the target distribution. By the change-of-variable formula \cite{kobyzev2020normalizing}, the probability density function associated with an invertible map $G_\theta$ is 
$$
p_\theta(x)=\rho\big(G_\theta^{-1}(x)\big) \abs{\text{det}\nabla G_\theta^{-1}(x)}.
$$
So the computation of the density function $p_\theta$ itself involves calculating the determinant of the Jacobians and does not allow the degeneracy of the Jacobian matrix.
However,  in our method, neither the Jacobian determinant nor invertibility is necessary, because the weak loss \eqref{eqn:oldloss} in WGS does not involve the expression of $p_\theta$, but rather only the {\it samples} from $p_\theta$ that are used.  

 Normalizing flows \cite{rezende2015variational,ruthotto2021introduction} are a family of invertible neural networks and provide a   way to construct the transport map $G$ as the composition of a sequence of functions $\{g_{\theta_i}\}_{i=1}^L$:
$$
G_\theta=g_{\theta_L}\circ g_{\theta_{L-1}}\circ\cdots\circ g_{\theta_1},
$$
parameterized by $\theta=\left(\theta_1,\theta_2,\dots,\theta_L\right)$ and $x=G_\theta(z)$. These functions serve to gradually transform the sample data point $z$ from the base distribution $\rho$ into the sample data point $x$ following the target distribution $p$. Explicitly, we have
$$
x^{(i)}=g_{\theta_i}\left(x^{(i-1)}\right),\quad i\in\left\{1,2,\cdots,L\right\},
$$
where $x^{(0)}=z$ and $x^{(L)}=x$. One popular example for constructing the parameterized function $g_{\theta_i}$ is via the triangular form and shuffled with each other by random coordinates. In this paper, we use a simple yet expressive affine coupling layer in the RealNVP \cite{dinh2016density,teshima2020coupling}. In RealNVP,  the affine coupling layer $g_{\theta_i}$ is defined as 
$$
x^{(i)}=g_{\theta_i}\left(x_1^{(i-1)},x_2^{(i-1)}\right)=\left(h_i\left(x_1^{(i-1)};\Theta_i\left(x_2^{(i-1)}\right)\right),x_2^{(i-1)}\right),\quad i\in\left\{1,2,\cdots,L\right\},
$$
where $(x_1^{(i-1)},x_2^{(i-1)})\in\mathbb{R}^a\times\mathbb{R}^{d-a}$ is the partition \footnote{This partition of $x^{(i)}$ can be randomized in practice.} of $x^{(i-1)}$,  and $\Theta_i:\mathbb{R}^{d-a}\mapsto\mathbb{R}^a$ is parameterized by $\theta_i=(\theta^1_i,\theta^2_i)$. Here, $h_i:\mathbb{R}^d\mapsto\mathbb{R}^d$ is the coupling function  that is defined as
$$
h_i\left(x_1^{(i-1)};\Theta_i\left(x_2^{(i-1)}\right)\right)=\left(x_1^{(i-1)}-t_{\theta^1_i}\left(x_2^{(i-1)}\right)\right)\odot\exp\left(-s_{\theta^2_i}\left(x_2^{(i-1)}\right)\right).
$$
Here, $t_{\theta^1_i}:\mathbb{R}^{d-a}\mapsto\mathbb{R}^a$ and $s_{\theta^2_i}:\mathbb{R}^{d-a}\mapsto\mathbb{R}^a$ are the translation and scaling functions, respectively. These functions are parameterized by neural networks with parameters $\theta^1_i$ and $\theta^2_i$, respectively.

 Since the affine coupling layer $g_{\theta_i}$ only updates $x_1^{(i-1)}$, we can update $x_2^{(i)}$ in the subsequent affine coupling layer $g_{\theta_{i+1}}$. This ensures that all components in $z$ are updated after transformation by the transport map $G_\theta$. This allows for more flexibility in the partition of $x^{(i)}$ and the arrangement of the affine coupling layers. Other popular coupling layers include splines and the mixtures of cumulative distribution functions, which have been shown to be more expressive than the affine coupling function \cite{durkan2019neural, ho2019flow++, jaini2019sum, tang2020deep}. 
In principle, any network structure for the map 
can work effectively for our weak generative sampler.
The details about the improvement of network architecture and expressive 
power are beyond the scope of our work here.

\subsection{Training algorithm}\label{s24}

\begin{algorithm}
\caption{Training Algorithm of WGS}\label{algorithm}
\SetKwInOut{Input}{Input}\SetKwInOut{Output}{Output}
\Input{Initial flow map $G_\theta$, the base distribution $\rho$; the  hyper-parameters $\gamma>0$, $\kappa>0$, $\lambda>0$, $r>0$, $c>0$. }

\For{$n=1:N_I$}{Sample $\{z_i\}_{i=1}^N$ from  $\rho$\;
Obtain $\{x_i\}_{i=1}^N$ by $x_i=G_\theta(z_i)$\;
Randomly choose $N_\varphi$ numbers from $1:N$ as index $ind$\;
Split $ind$ into mini-batches of size $N_\varphi^b$\;
\For{$m=1: \lceil N_\varphi/N_\varphi^b \rceil $}{
Obtain $\{x_{(j)}\}_{j=1}^{N_\varphi^b}$ by $x_{(j)}=x_{ind(m,j)}+\gamma \mathcal{N}(0,\boldsymbol{I}_d)$\;
\tcp*[h]{$\mathcal{N}(0,\boldsymbol{I}_d)$ denotes the standard $d$-dimensional normal random variables}\;
Construct the test function $\varphi_j$ by Gaussian kernel as
$$
\varphi_j(x)=\exp{-\frac{1}{2\kappa^2}(x-x_{(j)})^\top(x-x_{(j)})},
$$
\tcp*[h]{The parameter $\kappa$ denotes the standard deviation in each dimension}\;
Compute the Loss function \eqref{eqn:loss}\;
Update the parameters $\theta$ using the Adam optimizer with a learning rate $\eta$\;
}
}
\Output{The trained transport map $G_\theta$}
\end{algorithm}

 We now show in detail the training algorithm of our method. Since we are solving the SFPE on the whole $\mathbb{R}^d$ space, it is important in practice to restrict the map from pushing all points to infinitely far away, since any constant function satisfies $\mathcal{L}p=0$. We propose to constrain the range of the map within a ball  with a large radius of $r$ by  adding a  penalty term $L_b$ to the loss \eqref{eqn:oldloss}:
\begin{equation}\label{eqn:loss}
    L(G_\theta)=\frac{1}{N_\varphi^b}\sum_{j=1}^{N_\varphi^b}\left[\frac{1}{N}\sum_{i=1}^N \mathcal{L}^*\varphi_j\big(G_\theta(z_i)\big)\right]^2+\lambda L_{b},
\end{equation}
where 
$$
L_b=\frac{1}{N}\sum_{i=1}^N \text{Sigmoid}\Big(c\big(\|G_\theta(z_i)-x_0\|_2^2-r^2\big)\Big)
$$
and  the positive numbers
$\lambda$, $r$,  and $    c$ are hyper-parameters and $\text{Sigmoid}(x) = 1/(1+\exp(-x))$.
More specifically, we apply a large penalty to samples outside of the ball $B_r(x_0)$; the parameters $c$ and $\lambda$ essentially control the extent to which we want the push-forward probability measure $G_\theta\#\rho$ to be confined within the region $B_r(x_0)$.

Algorithm \ref{algorithm} provides the the complete procedure of our method.  There are two 
 important hyper-parameters associated with  the test functions: 
  $\kappa$   and   $\gamma$.    We briefly discuss the fine-tuning of these hyper-parameters here and refer to the numerical results section and Appendix for details. 
For $\kappa$,   the recommended  fine-tuning is to start with a relative large value for  sufficient exploration, and then shrinks gradually for good exploitation  to improve accuracy.
Multiple groups for different scheduling  $\kappa$ are also recommended; see Appendix \ref{appendix:d} and Appendix \ref{appendix:e}.  

The role of $\gamma$ is to enhance the exploration of  centers in the test function, so that the distribution of the centers in test functions can slightly shift from the target distribution; this is particularly useful when the target distribution may not be the best distribution of adaptive training.   A  non-zero $\gamma$  in general can produce better results than $\gamma=0$.

\subsection{Summary}
We shall demonstrate the efficiency and robustness of the WGS through several numerical examples in Section \ref{s4}, comparing it with  
 the benchmark  (weighted)-PINN loss in the ADDA method \cite{tang2022adaptive}. Efficiency refers to the significantly reduced computational time with the same number of training samples. Robustness refers to the ability to handle multi-modal invariant distributions associated with the SDE, as well as stability with respect to the hyper-parameters tested in the Appendix \ref{appendix:e}.
The enhanced   performance in capturing multi-modes arises from the test functions, which can ``sense'' a neighborhood of size $\kappa$ around each training sample. In addition, Appendix \ref{appendix:b} presents an analysis of a toy example with bi-modes, comparing the loss landscapes    and  providing  certain insights into the role of the test functions in the WGS.

\section{\texorpdfstring{$L^2$}{} Error Estimate  of WGS}\label{s3}

As mentioned earlier,    
our weak generative sampler 
addresses the  optimization problem \eqref{eqn:wsfpe_e}
by  minimizing  the following   loss function derived from the weak form of the stationary  Fokker--Planck equation \eqref{eqn:sfpe},
\begin{equation}\label{optN}
 \int_{\Omega}\left| \int_{\mathbb{R}^d}\mathcal{L}^*\varphi(x)p(x)\d x \right|^2\ \dd \mathbb{P}(\varphi) ,
\end{equation}
which can be written in the expectation form $\mathbb{E}_{\varphi\sim\mathbb{P} }\left| \mathbb{E}_{x\sim p}\mathcal{L}^*\varphi(x)\right|^2. $ Here $\mathbb{P}$ can be any   non-degenerate  probability measure 
on  the space of test functions $\varphi$. As explained in Section \ref{s22}, 
 the   requirement for non-degeneracy of $\mathbb{P}$
is   the full support property to ensure    that the zero loss value of \eqref{optN} can imply Equation 
\eqref{eqn:wsfpe} for the minimizer $p$.

In principle,  any non-degenerate $\mathbb{P}$ can be used for the randomized test functions.  In Algorithm \ref{algorithm}, we adopt a choice of $\mathbb{P}$ from  the family of Gaussian kernel functions which is adaptively determined based on the training samples generated by the current map. The advantages of this approach in the WGS framework will be demonstrated later in Section \ref{s4}.
At this point, we present a {\it prior} estimate of the $L^2$ error between the true and numerical probability density functions in Theorem \ref{theo:R}. It is important to note that the results in this section are unrelated to the relaxation techniques used for resolving the worst-case min-max issue. Additionally, the theoretical error analysis provided here does not prescribe the practical selection of test functions for the algorithms.

Under the technical Assumptions \ref{assum:bounded} and \ref{assumptionptheta}, which ensure the non-degeneracy condition for $\mathbb{P}$, the $L^2$ error $\|p_\theta-p\|$ is shown to be bounded by the weak loss associated with a {\it specific} distribution of test functions derived from the true error $p_\theta-p$, up to an arbitrarily small adjustment of the PINN loss.
To highlight the core idea of our proof, we first focus on the Fokker-Planck equation with periodic boundary conditions on a hypercube in Section \ref{ss:bd}. This setting helps us avoid the intricate technical problems associated with compactness that usually arise from boundary conditions at infinity. Subsequently, the proof for the whole space $\mathbb{R}^d$ is developed in Section \ref{ss:R}.

\subsection{Stationary Fokker--Planck equations in  periodic   domain} \label{ss:bd}
For given positive constants $\{R_i\}_{i=1}^d$, let $U:=\prod_{i=1}^d(0,R_i)\subset\mathbb{R}^d$. In the Fokker--Planck operator $\mathcal{L}$  defined in \eqref{eqn::fp}, the drift term $b$ and diffusion matrix $\sigma$ are assumed to be $U$-periodic, i.e., they admit period $R_i$ in the $i$-th direction, $i=1,\cdots,d$. We assume that $b\in C^1_{per}(U)$ and $D\in C^2_{per}(U)$,
 and assume that  $p$ is the unique classic nontrivial solution to  the  stationary Fokker--Planck equation 
on the hypercube $U$,
\begin{equation}\label{eqn:sfpeU}
     \mathcal{L}p=0,\quad \mbox{in}\quad U,
\end{equation}
and $p$ satisfies the periodic boundary condition and normalization condition $\int_U p(x)\dd x =1$. 
The  solution $p$ can  also be interpreted as a weak function in 
the periodic Sobolev space $H^1_{\text{per}}(U)$, the Sobolev space consisting of $U$-periodic functions whose first weak derivatives exist and $L^2$ integrable on $U$, so $p$   satisfies  
\begin{equation}%\label{equ:sfpeUweak}
    \int_{U}\mathcal{L}^*\varphi(x)p(x) \dd x=0,\quad\forall \varphi\in \Omega_0:=C^\infty_0(U ).
\end{equation}

In the numerical  computation of minimizing  \eqref{eqn:wsfpe_e}, the probability density function 
belongs to a family parametrized by a generic parameter $\theta$ in a specific set $\Theta$: 
  $\{p_\theta\}_{\theta\in \Theta}$.  This following assumption is trivially fulfilled in our setting here.
\begin{assumption}\label{subassum:pM}
 We make the following assumptions about the family of the density function $p_\theta$ that for all $\theta\in\Theta$
    \begin{equation*}
        \begin{aligned}
        0\leq p_\theta(x)\leq M,\quad \forall x\in U; \quad \int_{U}p_\theta(x)\d x=1.
        \end{aligned}
\end{equation*}
where $M$ is a constant.
\end{assumption}

Our result about the $L^2$ error and the loss \eqref{optN} needs the following assumption of the probability $\mathbb{P}$ on the space of the test function $\Omega_0=C^\infty_0(U)$. 
The rigorous statement about the existence and construction of such probability measures  $\mathbb{P}$ is shown in the Appendix \ref{appendix:a}.
\begin{assumption}\label{assum:bounded}
  \begin{enumerate}
   \item  For any positive number $r$ and $f\in L^2(U)$ it holds that 
   $$ \mathbb{P}\left(\bar{B}_{L^2(U)}(f,r)\cap C_0^\infty(U ) \right)>0,
$$
    where $\bar{B}_{L^2(U)}(f,r)=\{g\in L^2(U ):\ \|f-g\|_{L^2(U)}\leq r \}$ is the closed ball centered at $f$   with $r$  radius  in $L^2(U)$.\label{subass:L2}
 \item For any positive number $r$ and $f\in H^2(U)$, it holds that 
        \begin{equation*}
        \mathbb{P}\left(\bar{B}_{H^2(U)}(f,r)\cap C_0^\infty(U ) \right)>0,
    \end{equation*}
    where $\bar{B}_{H^2(U)}(f,r)=\{g\in H^2(U ):\ \|f-g\|_{H^2(U)}\leq r \}$ is the closed ball centered at $f$   with $r$  radius  in $H^2(U)$.\label{subass:H2}
 \end{enumerate}
\end{assumption}

Our theorem below asserts that with a properly chosen test function, the weak loss closely approximates the mean square loss relative to the true solution.
\begin{theorem}\label{theo:U}
Let $U\subset \mathbb{R}^d$ be a hypercube with periodic conditions.
Let $p\in C^2_{\text{per}}(U )$ be the classical solution of the stationary Fokker--Planck equation \eqref{eqn:sfpeU} on the hypercube $U$. Suppose that $p_{\theta}\in C^2_{\text{per}}(U )$ for  every $\theta\in\Theta$, and  Assumption \ref{subassum:pM} holds.
For  any $\theta\in\Theta$, define 
  $\varphi_{\theta}$ as  the solution to the 
Dirichlet boundary value problem
\begin{equation}\label{dirichletproblem:U}
\begin{cases}
\mathcal{L}^*\varphi_{\theta}=\ p_{\theta}-p, & \mbox{in}\ U,\\
    \varphi_{\theta} \quad =\ 0,&\mbox{on}\ \partial U.
\end{cases}
\end{equation}
Let $\mathbb{P}_{\theta, r}(\cdot)$, $\mathbb{P}'_{\theta, r}(\cdot)$ denote  the conditional distributions $\mathbb{P}(\cdot\ |\  \bar{B}_{L^2(U)}(\varphi_\theta,r))$ and $\mathbb{P}(\cdot\ |\   \bar{B}_{H^2(U)}(\varphi_\theta,r))$ respectively.

     \begin{enumerate}[label=(\alph*)]
        \item If Assumption \ref{assum:bounded}-\eqref{subass:L2} holds, then for any $r>0$,  \begin{equation}\label{upper:strong} 
          \begin{aligned}
              \mathbb{E}_{\varphi\sim\mathbb{P}_{\theta, r} }\left| \mathbb{E}_{x\sim p_{\theta }}\mathcal{L}^*\varphi(x)\right| -  r   \|\mathcal{L}p_\theta\|_{L^2(U)} 
              \leq 
              \|p_\theta-p\|^2_{L^2(U)}\leq&\ \mathbb{E}_{\varphi\sim\mathbb{P}_{\theta, r} }\left| \mathbb{E}_{x\sim p_{\theta }}\mathcal{L}^*\varphi(x)\right| +  r  \|\mathcal{L}p_\theta\|_{L^2(U)},
          \end{aligned}
               \end{equation}

        \item  If Assumption \ref{assum:bounded}-\eqref{subass:H2} holds, then for any $r>0$,  \begin{equation}\label{firstupperbound}
          \begin{aligned}
             \mathbb{E}_{\varphi\sim\mathbb{P}'_{\theta, r} }\left| \mathbb{E}_{x\sim p_{\theta }}\mathcal{L}^*\varphi(x)\right| -  r   C  \leq \|p_\theta-p\|^2_{L^2(U)}\leq&\ \mathbb{E}_{\varphi\sim\mathbb{P}'_{\theta, r} }\left| \mathbb{E}_{x\sim p_{\theta }}\mathcal{L}^*\varphi(x)\right| +  r   C ,
          \end{aligned}
              \end{equation}
    \end{enumerate}
where  
\begin{equation} \label{C551}
     \begin{aligned}
         C  =&\ M\left(d\max_{i}\|b_i\|_{U} + d^2  \max_{i,j}\|D_{ij}\|_{U} \right).
     \end{aligned}
\end{equation}
  \end{theorem}

\begin{remark}
  Note that  a measurable functional $F$,  the conditional expectation  $ \mathbb{E}_{\varphi\sim\mathbb{P}_{\theta, r} } F(\varphi)= \mathbb{E}_{\varphi\sim\mathbb{P} } \left[F(\varphi) \mathbf{1}_{\bar{B} (\varphi_\theta,r)} (\varphi)\right] / \mathbb{P}(\bar{B} (\varphi_\theta,r))
  \le  \mathbb{E}_{\varphi\sim\mathbb{P} } \left[F(\varphi)  \right] / \mathbb{P}(\bar{B} (\varphi_\theta,r))$. So, the above upper bounds can be replaced by $\mathbb{E}_{\varphi\sim\mathbb{P} }\left| \mathbb{E}_{x\sim p_{\theta }}\mathcal{L}^*\varphi(x)\right| / \mathbb{P}(\bar{B} (\varphi_\theta,r)) $. 
 In addition, according to Jensen's inequality, the quadratic form 
    of our loss function \eqref{optN} is the upper bound of the term $\mathbb{E}_{\varphi\sim\mathbb{P} }\left| \mathbb{E}_{x\sim p_{\theta }}\mathcal{L}^*\varphi(x)\right|$:
$
   \left(\int_{\Omega_0} \left|\int_{U}\mathcal{L}^*\varphi(x) p(x) \d x\right|\d\mathbb{P} (\varphi)\right)^2 \leq \int_{\Omega_0} \left|\int_{U}\mathcal{L}^*\varphi(x) p(x) \d x\right|^2\d\mathbb{P} (\varphi).
$
\end{remark}

 \begin{remark}
In the limit of $r\to 0$,  $\mathbb{E}_{\varphi\sim\mathbb{P}_{\theta, r} }\left| \mathbb{E}_{x\sim p_{\theta }}\mathcal{L}^*\varphi(x)\right|$ tends to $   \left| \mathbb{E}_{x\sim p_{\theta }}\mathcal{L}^*\varphi_\theta(x)\right| $
which  exactly recovers the   squared $L^2$ error $\|p-p_\theta\|^2_{L^2(U)}$. In general, for a finite $r$, 
 our results heuristically suggest that it is desirable for the probability $\mathbb{P}$ of sampling the test function to adaptively primarily concentrate within a small ball around $\varphi_\theta$.
Note that $\varphi_\theta$ can be regarded as an {\it a prior} estimation of the error between $p_\theta$ and the solution $p$. \end{remark}

\begin{proof}
Since $p_\theta$ and the true solution $p$ both in  $C^2_{per}(U)$, 
the Dirichlet problem \eqref{dirichletproblem:U} has a unique $C^{2,1}_{per}(U )$ solution $\varphi_\theta$. 
Then we can  write the 
squared $L^2$ error between $p_\theta$ and $p$  in terms of $\varphi_\theta$:
\begin{equation} \label{615}
\begin{aligned}
\int_{U}|p_{\theta}(x)-p(x)|^2 \d x 
    =&\ \int_{U}
{\mathcal{L}^*\varphi_{\theta}(x)} (p_{\theta}(x)-p(x)) \d x 
    = \int_{U}\mathcal{L}^*\varphi_{\theta}(x)p_{\theta}(x) \d x,
\end{aligned}
\end{equation}
by using  the fact that $p$ is the true solution satisfying the periodic boundary condition  and $\varphi_\theta$ vanishes on $\partial U$.

 Let $r$ be a positive constant. Use a concise notation    $\Omega_{\theta, r  }:=\{\varphi\in C_0^\infty(U) :\ \| \varphi-\varphi_{\theta }\|_{L^2(U)}\leq  r  \}\subset \Omega_0 $  to represent    the $L^2$-ball.
 By Assumption \ref{assum:bounded} we know that  the conditional distribution 
 $\mathbb{P}_{\theta, r}( \cdot)=1$ is well-defined and $\mathbb{P}_{\theta, r}(\Omega_{\theta,r})=1$. We now obtain the bound for the squared $L^2$ error $\|p_\theta-p\|^2_{L^2(U)}$ below. Note that,
\begin{equation}\label{ineqn:L2}
\begin{aligned}
    & \int_{\Omega_0}\left| \int_{U}\mathcal{L}^*\varphi(x)p_{\theta }(x) \d x \right|\ \dd\mathbb{P}_{\theta, r} (\varphi)
    = \
    \int_{\Omega_{\theta, r  }}\left| \int_{U}\left(\varphi(x) -\varphi_{\theta}(x) + \varphi_{\theta}(x)\right)\mathcal{L}p_{\theta }(x) \d x \right|\ \dd\mathbb{P}_{\theta, r}(\varphi)\\
    \geq &\ \int_{\Omega_{\theta, r  }}\left| \int_{U}\varphi_{\theta}(x)\mathcal{L}p_{\theta }(x) \d x \right|\ \dd\mathbb{P}_{\theta, r}(\varphi) - \int_{\Omega_{\theta, r  }} \left(\int_{U}\left|\varphi(x)-\varphi_{\theta}(x)\right|^2 \d x \right)^{1/2}\left(\int_{U}|\mathcal{L}p_{\theta }(x)|^2 \d x \right)^{1/2}\ \dd\mathbb{P}_{\theta, r} (\varphi)\\ 
    \geq&\ \left| \int_{U}\varphi_{\theta}(x)\mathcal{L}p_{\theta }(x) \d x \right|\mathbb{P}_{\theta, r} (\Omega_{\theta, r  }) -   r  \|\mathcal{L}p_\theta\|_{L^2}\mathbb{P}_{\theta, r} (\Omega_{\theta, r  }),
    \\
    = &  \int_{U}|p_{\theta}(x)-p(x)|^2 \d x   -   r  \|\mathcal{L}p_\theta\|_{L^2(U)}.
\end{aligned}
\end{equation} 
This implies the second inequality ``$\leq$''  as the upper bound in \eqref{upper:strong}. We can prove the first $``\leq"$ part (lower bound) in a similar way by revising the second line in the above derivation \eqref{ineqn:L2} by using $\abs{\varphi(x) -\varphi_{\theta}(x) + \varphi_{\theta}(x)} \leq \abs{ \varphi_{\theta}(x)}+\abs{\varphi(x) -\varphi_{\theta}(x)}$. 

The proof for the second statement \eqref{firstupperbound}   follows a similar approach to the one above, albeit with some minor differences.
 Define the $H^{2}$-ball   $ \Omega'_{\theta, r  }:=\{\varphi\in C_0^\infty(U) :\ \| \varphi-\varphi_{\theta }\|_{H^2(U)}\leq  r  \}.$ 
 Then we have \begin{equation}\label{ineqn:H2}
\begin{aligned}
& \mathbb{E}_{\varphi\sim\mathbb{P}'_{\theta, r} }\left| \mathbb{E}_{x\sim p_{\theta }}\mathcal{L}^*\varphi(x)\right|
= \
\int_{\Omega'_{\theta, r  }}\left| \int_{U}\mathcal{L}^*\left(\varphi(x) - \varphi_{\theta}(x) + \varphi_{\theta}(x)\right)p_{\theta }(x)\d x \right|\ \dd\mathbb{P}'_{\theta, r} (\varphi)
\\
\geq &\ \left| \int_{U}\mathcal{L}^*\varphi_{\theta }(x)p_{\theta }(x)\d x \right|\ \mathbb{P}_{\theta, r}(\Omega'_{\theta, r  }) - \int_{\Omega'_{\theta, r  }}\int_{U} \left| \sum_{i=1}^d b_i(x)\left(\partial_{x_i} \varphi(x)-\partial_{x_i}\varphi_{\theta}(x)\right) \right.\\
&\ \left.  + \sum_{i,j=1}^d D_{ij}(x)(\partial^2_{x_jx_i}\varphi(x)- \partial^2_{x_jx_i}\varphi_{\theta}(x)) \right|p_\theta(x) ~\dd x~\dd\mathbb{P}'_{\theta, r} (\varphi)\\ 
\geq &\ 
\left| \int_{U}\mathcal{L}^*\varphi_{\theta }(x)p_{\theta }(x)\d x \right|\ \mathbb{P}'_{\theta, r}(\Omega_{\theta, r  })
- \int_{\Omega'_{\theta, r  }} \|\varphi-\varphi_{\theta}\|_{H^2(U)}\\
&\ \times \left[ \sum_{i=1}^d\left(\int_{U}|b_i(x)p_{\theta }(x)|^2\d x\right)^{1/2}  + \sum_{i,j=1}^d \left(\int_{U}|D_{ij}(x)p_{\theta }(x)|^2\d x\right)^{1/2} \right] \dd \mathbb{P}'_{\theta, r} \\ 
\geq &\ \int_{U}|p_{\theta}(x)-p(x)|^2 \d x -   r   C ,
\end{aligned}
\end{equation} 
where 
 $  C $ is defined in \eqref{C551}. The proof for the converse inequality side is the same as that for the first statement.
  \end{proof}

\subsection{Stationary Fokker--Planck equations in   \texorpdfstring{$\mathbb{R}^d$}{}}\label{ss:R}

 Under Assumption \ref{assum:SDE}, 
 the SDE \eqref{eqn:sde} is ergodic
 and the Fokker--Planck equation has a unique   invariant probability density function  on the entire space $\mathbb{R}^d$ which decays at infinity. 
  The similar results to Theorem \ref{theo:U} can be derived.
  
   For simplicity, we assume the drift term and the functions in the diffusion matrix are bounded.
\begin{assumption}\label{boundedbsigma}
    Suppose $\|b_i\|_\infty<\infty$ and $\|\sigma_{ij}\|_\infty<\infty$ for $i,j\in\{1,\cdots,d\}$.
\end{assumption}
 
We list the assumptions concerning the family of probability density functions $\{p_\theta\}_{\theta\in\Theta}$ and the probability measure $\mathbb{P}$ on the test function space $\Omega=C_c^\infty(\mathbb{R}^d)$.

\begin{assumption}\label{assumptionptheta}
\hfill
    \begin{enumerate}[label=(\arabic*), leftmargin=1cm]
        \item There exists a positive constant  $M$ such that for every $\theta\in\Theta$,
    \begin{equation*}
        \begin{aligned}
            0\leq p_\theta(x)\leq M,\quad  \forall x\in \mathbb{R}^d,\quad \int_{\mathbb{R}^d}p_\theta(x)\d x=1.
        \end{aligned}
        \end{equation*}\label{subassum:RpM}
        \item\label{assumptiononp} The family functions $\{p_\theta\}_{\theta\in \Theta}$ and $p$ are uniformly tight, i.e., for every $\varepsilon>0$ there exists a compact subset $U_\varepsilon\subset\mathbb{R}^d$ such that for all $\theta\in\Theta$ and $p$
        \begin{equation*}
            \int_{U_\varepsilon^c}p_\theta(x)\d x<\varepsilon, \quad \int_{U_\varepsilon^c}p(x)\d x<\varepsilon.
        \end{equation*}
        Without loss generality, for such given $\varepsilon$ we assume there exists a positive constant $r_\varepsilon$  such that $U_\varepsilon=\bar{B}_{r_\varepsilon}$, the closed ball centered at 0 with radius $r_\varepsilon$ in $\mathbb{R}^d$. In addition, we assume
        \begin{equation*}
            0\leq p_\theta(x),\ p(x)<1,\quad \forall x\in B^c_{r_\varepsilon},\ \theta\in\Theta.
        \end{equation*}\label{subassum:Rvarepsilon}
    \item The probability measure $\mathbb{P}$ satisfies that for any given number $r >0$ and $f\in L^2(\mathbb{R}^d)$, it holds that 
       $$
\mathbb{P}\left(\bar{B}_{L^2(\mathbb{R}^d)}(f,r)\cap C_c^\infty(\mathbb{R}^d ) \right)>0,
$$
    where $\bar{B}_{L^2(\mathbb{R}^d)}(f,r)=\{g\in L^2(\mathbb{R}^d):\ \|f-g\|_{L^2(\mathbb{R}^d)}\leq r\}$ is the closed ball centered at $f$ and with radius $r$   in $L^2(\mathbb{R}^d)$.\label{subassum:L2R}
    \item The probability measure $\mathbb{P}$ satisfies that  for any given positive number $r >0$ and $f\in H^2(\mathbb{R}^d)$, it holds that
        \begin{equation*}
        \mathbb{P}(\bar{B}_{H^2(\mathbb{R}^d)}(f,r)\cap C_c^\infty(\mathbb{R}^d))>0,
    \end{equation*}
    where $\bar{B}_{H^2(\mathbb{R}^d)}(f,r)=\{g\in H^2(\mathbb{R}^d):\ \|f-g\|_{H^2(\mathbb{R}^d)}\leq r_0  \}$ is the closed ball centered at $f$ and with radius $r$   in $H^2(\mathbb{R}^d)$.\label{assumptionprobabilitmeasure}
    \end{enumerate}
\end{assumption}

\begin{theorem}\label{theo:R}
Let $p\in C^2(\mathbb{R}^d)$ be the classical solution of the stationary Fokker--Planck equation \eqref{eqn:sfpe}. Suppose that $p_{\theta}\in C^2(\mathbb{R}^d)$ for  every $\theta\in \Theta$, Assumption \ref{boundedbsigma} and  \ref{assumptionptheta}-\ref{subassum:RpM}-\ref{subassum:Rvarepsilon} hold. For any   $\varepsilon>0$, let $\varphi_{\theta,\varepsilon}$ be the solution to the boundary value 
Dirichlet problem on the ball $B_{r_\varepsilon}$,
\begin{align}\label{dirichletproblem}
\left\{
\begin{aligned}
    \mathcal{L}^*\varphi_{\theta,\varepsilon }&=\ p_{\theta}-p,\ &\mbox{in}\ B_{r_\varepsilon },\\
    \varphi_{\theta,\varepsilon }&=\ 0,\ &\mbox{on}\ \partial B_{r_\varepsilon },
    \end{aligned}\right.
\end{align}
and extend the definition   $ {\varphi}_{\theta,\varepsilon}=0$ in $B^c_{r_\varepsilon}$.
Let $\mathbb{P}_{\theta, r,\varepsilon}(\cdot)$, $\mathbb{P}'_{\theta, r,\varepsilon}(\cdot)$ be the conditional distributions $\mathbb{P}(\cdot\ |\  \bar{B}_{L^2(\mathbb{R}^d))}( \varphi_{\theta,\varepsilon},r))$ and $\mathbb{P}(\cdot\ |\   \bar{B}_{H^2(\mathbb{R}^d)}( \varphi_{\theta,\varepsilon},r))$ respectively.
\begin{enumerate}[label=(\alph*)]
\item
     If Assumption \ref{assumptionptheta}-\ref{subassum:L2R} holds,   then  for any $ r >0$,
        \begin{equation}\label{firstboundR}
             \mathbb{E}_{\varphi\sim\mathbb{P}_{\theta, r,\varepsilon} }\left| \mathbb{E}_{x\sim p_{\theta }}\mathcal{L}^*\varphi(x)\right| -  r   \|\mathcal{L}p_\theta\|_2               \leq 
              \|p_\theta-p\|^2_2\leq \ \mathbb{E}_{\varphi\sim\mathbb{P}_{\theta, r,\varepsilon} }\left| \mathbb{E}_{x\sim p_{\theta }}\mathcal{L}^*\varphi(x)\right| +  r  \|\mathcal{L}p_\theta\|_2 + \varepsilon.
    \end{equation}
    \item 
If Assumption \ref{assumptionptheta}-\ref{assumptionprobabilitmeasure}  holds, then for any $ r >0$,

    \begin{equation}\label{secondboundR}
           \begin{aligned}
            \mathbb{E}_{\varphi\sim\mathbb{P}'_{\theta, r,\varepsilon} }\left| \mathbb{E}_{x\sim p_{\theta }}\mathcal{L}^*\varphi(x)\right| -  r   C  \leq \|p_\theta-p\|^2_2\leq\ \mathbb{E}_{\varphi\sim \mathbb{P}'_{\theta, r, \varepsilon}}  \left| \mathbb{E}_{x\sim p_{\theta }}\mathcal{L}^*\varphi(x)\right| +  r   C + \varepsilon,
           \end{aligned}
              \end{equation}
\end{enumerate}
where 
\begin{equation}
     \begin{aligned}
         C=&\ M\left[d\max_{i}\|b_i\|_{\infty} + d^2  \max_{i,j}\|D_{ij}\|_{\infty} \right].\label{C7}
     \end{aligned}
\end{equation}
\end{theorem}
 
\begin{proof}
According to Assumption \ref{assumptionptheta}-\ref{assumptiononp}, we observe that the $L^2$ error primarily concentrates within the bounded domain $B_{r_\varepsilon}$,
\begin{equation}\label{varepsilonp}
\begin{aligned}
    \int_{B_{r_\varepsilon }}|p_{\theta}(x)-p(x)|^2 \d x
    \le &   \int_{\mathbb{R}^d}|p_{\theta}(x)-p(x)|^2 \d x  \\
    =&\ \int_{B_{r_\varepsilon }}|p_{\theta}(x)-p(x)|^2 \d x + \int_{B^c_{r_{\varepsilon }}}|p_{\theta}(x)-p(x)|^2 \d x \\
    \leq&\ \int_{B_{r_\varepsilon }}|p_{\theta}(x)-p(x)|^2 \d x + \varepsilon.
\end{aligned}
\end{equation}
And $p_\theta-p$ is at least $C^1$ in $\mathbb{R}^d$, making
the Dirichlet problem \eqref{dirichletproblem} have a unique $C^{2,1}(\bar B_{r_\varepsilon })$ solution (see \cite[Theorem 6.8 \& Corollary 6.9]{gilbarg1977elliptic}). Since  $ \varphi_{\theta,\varepsilon }$ may not  be in   $C^2(\mathbb{R}^d)$,  we consider its $\delta$-mollification  
$$\widehat\varphi_{\theta,\varepsilon ,\delta}(x):=\int_{\mathbb{R}^d}\eta_\delta(x-y) \widehat\varphi_{\theta,\varepsilon }(y)\d y=\int_{B_\delta}\eta_\delta(y) \widehat\varphi_{\theta,\varepsilon }(x-y)\d y,$$
where 
$\eta_\delta=\frac{1}{\delta^d}\eta\left(\frac{x}{\delta}\right)$,
and $\eta$ is the standard mollifier
\begin{equation*}
    \eta(x)=\begin{cases}
        K_0\exp\left(\frac{1}{|x|^2-1}\right),\ \mbox{if}\ |x|<1,\\
        0,\ \mbox{if}\ |x|\geq 1,
    \end{cases}
\end{equation*}
and $K_0>0$ is a constant selected such that $\int_{\mathbb{R}^d}\eta \d x=1$. It is easy to see \cite[Appendix C4 \& Section 5.3.1 Theorem 1]{evans2022partial} that $\widehat\varphi_{\theta,\varepsilon ,\delta}\in C_c^\infty(\mathbb{R}^d)$ and 
$$\widehat\varphi_{\theta,\varepsilon ,\delta}\to  \varphi_{\theta,\varepsilon },\quad\mbox{in}\quad H^{2}_{\mathrm{loc}}(\mathbb{R}^d),\quad \mbox{as}\quad\delta\to0.$$

Now we can estimate the squared $L^2$ error in terms of the special test function $\widehat{\varphi}_{\theta,\varepsilon,\delta}\in C_c^\infty(\mathbb{R}^d)$  as follows
by noting $\int_{\mathbb{R}^d}\mathcal{L}^*\widehat\varphi_{\theta,\varepsilon ,\delta}(x)p(x) \d x=0$ for the true solution $p$ and  $\varphi_{\theta,\varepsilon }(x)=\widehat\varphi_{\theta,\varepsilon ,\delta}(x)=0$ outside of $B_{1+r_\epsilon}$:
\begin{equation*}
\begin{aligned} &   \int_{B_{r_\varepsilon }}|p_{\theta}(x)-p(x)|^2 \d x   
    = \ \int_{\mathbb{R}^d}\mathcal{L}^* \varphi_{\theta,\varepsilon }(x)(p_{\theta}(x)-p(x)) \d x  
    \\
    =&\ \int_{\mathbb{R}^d}\mathcal{L}^*\widehat\varphi_{\theta,\varepsilon ,\delta}(x)p_{\theta}(x) \d x + \int_{B_{r_\varepsilon +1}}\mathcal{L}^*(\varphi_{\theta,\varepsilon }(x)-\widehat\varphi_{\theta,\varepsilon ,\delta})(p_{\theta}(x)-p(x)) \d x   \\
    =:&\ \int_{\mathbb{R}^d}\mathcal{L}^*\widehat\varphi_{\theta,\varepsilon ,\delta}(x)p_{\theta}(x) \d x +  I_\delta ,
\end{aligned}
\end{equation*}
where  $I_\delta$ refers to the second item
in the second line above and note that  $
\abs{I_\delta}\le   C_\delta=\ \|p_\theta-p\|_{B_{r_\varepsilon +1}}\left(d\max_{i}\|b_i\|_{B_{r_\varepsilon +1}}  + d^2\max_{i,j}\|D_{ij}\|_{B_{r_\varepsilon+1}}\right)\|\widehat\varphi_{\theta,\varepsilon }-\widehat\varphi_{\theta,\varepsilon ,\delta}\|_{H^2(B_{r_\varepsilon +1})}$. Since $\lim_{\delta\to 0} C_\delta \to 0$ ,
then 
$\int_{B_{r_\varepsilon }} \abs{p_{\theta }(x)-p(x)}^2 \d x  
    = \  \int_{\mathbb{R}^d}\mathcal{L}^*\widehat\varphi_{\theta ,\varepsilon }(x)p_{\theta }(x) \d x.
$
Consequently, by \eqref{varepsilonp}, we proved that 
\begin{equation*}
\begin{aligned}
\int_{\mathbb{R}^d}\mathcal{L}^* \varphi_{\theta ,\varepsilon }(x)p_{\theta }(x) \d x
\le 
\int_{\mathbb{R}^d}|p_{\theta }(x)-p(x)|^2 \d x  
    \leq \int_{\mathbb{R}^d}\mathcal{L}^* \varphi_{\theta ,\varepsilon }(x)p_{\theta }(x) \d x + \varepsilon  .
\end{aligned}
\end{equation*}
which is analogous  to \eqref{615} in the proof of Theorem \eqref{theo:U}.
The remaining proofs are similar to that of 
which is analogous  to \eqref{615} in the proof of Theorem \eqref{theo:U}.

To prove the first statement, let $r$ be a positive constant and $\Omega_{\theta, r  }:=\{\varphi\in C_c^\infty(\mathbb{R}^d) :\ \| \varphi-\varphi_{\theta }\|_{L^2(U)}\leq  r  \}   $ represent    the $L^2$-ball. By Assumption \ref{assumptionptheta}-\ref{subassum:L2R} we know that  the conditional distribution 
 $\mathbb{P}_{\theta, r,\varepsilon}( \cdot)$ is well-defined and $\mathbb{P}_{\theta, r,\varepsilon}(\Omega_{\theta,r,\varepsilon})=1$. We now can obtain a bound for the squared $L^2$ error as follows. By replacing $U$ with $\mathbb{R}^d$, $\mathbb{P}_{\theta,r}$ with $\mathbb{P}_{\theta,r,\varepsilon}$, and $\Omega_{\theta, r}$ with $\Omega_{\theta, r,\varepsilon}$ in \eqref{ineqn:L2}, we immediately obtain the second inequality ``$\leq$'' of \eqref{firstboundR}. The first ``$\leq$'' part can be proven in a similar manner as outlined in Section \ref{ss:bd}.

Now we turn to the second statement. The proof will also follow a similar approach as outlined in Section \ref{ss:bd}. For any positive constant $r$, we select an $H^2$-ball: $\Omega'_{\theta, r, \varepsilon} := \{\varphi \in C_c^\infty(\mathbb{R}^d) : | \varphi - \widehat\varphi_{\theta, \varepsilon}|_{H^2(\mathbb{R}^d)} \leq r\}$, and replace $U$ with $\mathbb{R}^d$, $\mathbb{P}'_{\theta, r}$ with $\mathbb{P}'_{\theta, r, \varepsilon}$, and $\Omega'_{\theta, r}$ with $\Omega'_{\theta, r, \varepsilon}$ in \eqref{ineqn:H2}. Then we obtain \eqref{secondboundR} where the constant $C$ is given by \eqref{C7}.
\end{proof}

\section{Numerical Experiments}\label{s4}
In this section, we apply the WGS to  several  different examples: a two-dimensional system with a single mode, a two-dimensional system with two metastable states, a three-dimensional Lorenz system, and two high-dimensional problems. We use Real NVP, as mentioned in Section \ref{s23}, to parameterize the generative map $G$ in all the examples. For each affine coupling layer in Real NVP, we use the fully connected neural networks with three hidden layers and the LeakyReLU as the activation function to parameterize the translation and scaling functions. Unless specifically stated, we use the base distribution as the standard Gaussian distribution $\rho(z)=\mathcal{N}(z;0,\boldsymbol{I}_d)$. The hyper-parameters  for each examples  are provided in Table \ref{tab_hp} in the Appendix \ref{appendix:e}.

In the first and second examples, we compare the WGS with the ADDA  method  proposed in \cite{tang2022adaptive}, where the loss function is defined as 
\begin{equation}\label{eqn:addaloss}
    L_{\text{ADDA}}=\frac{1}{N_p}\sum_{i=1}^{N_p}|\mathcal{L}p_\theta(x_i)|^2+\lambda L_b,
\end{equation}
where $\lambda L_b$ represents the same boundary as in our method. In \eqref{eqn:addaloss}, $\{x_i\}_{i=1}^{N_p}$ is sampled by $p_\text{data}(x)$, where $p_\text{data}(x)$ is initially set as a uniform distribution in a bounded domain, and then set as the $p_\theta(x)$ as the adaptive technique   \cite{tang2022adaptive}. We use the same network structure as in the WGS to parameterize the generative map $G_\theta$ and $p_\theta=G_{\theta\#}\rho$.

To assess the accuracy of the learned invariant measure, we compute the relative error of the learned distribution $p_\theta(x)$ push-forward by  $G_\theta$ from $\rho(z)$:
$$
e_p=\frac{\|p_\theta(x)-p(x)\|_2}{\|p(x)\|_2}.
$$

\subsection{Example 1: A two-dimensional system with single mode}\label{s41}
To test the efficiency of WGS, we consider the following two-dimensional system
\begin{equation}
    \left\{\begin{array}{l}
\dx=-(x-1)\dt+\sqrt{2} \dW_1, \\
\dy=-(y-1)\dt+\sqrt{2} \dW_2.
\end{array}\right.
\end{equation}
The invariant distribution for this example is $p(x)=\exp(-(x-1)^2/2-(y-1)^2/2)/2\pi$.

For WGS, we generated $N=8000$ sample points from the base distribution. We selected $N_\varphi = 500$ test functions, whose means were sampled from the uniform distribution in the box $[-4,4]\times[-4,4]$. The scale parameter $\kappa$ of the test functions is $1.0$. For ADDA, we utilized $N_p=8000$ data points sampled from $p_\theta$ during the training process. We utilize the Adam optimizer with a decay weight of learning rate to train WGS and ADDA for $10000$ iterations.

\begin{figure}[ht]
    \centering
    \includegraphics[width=.9\textwidth]{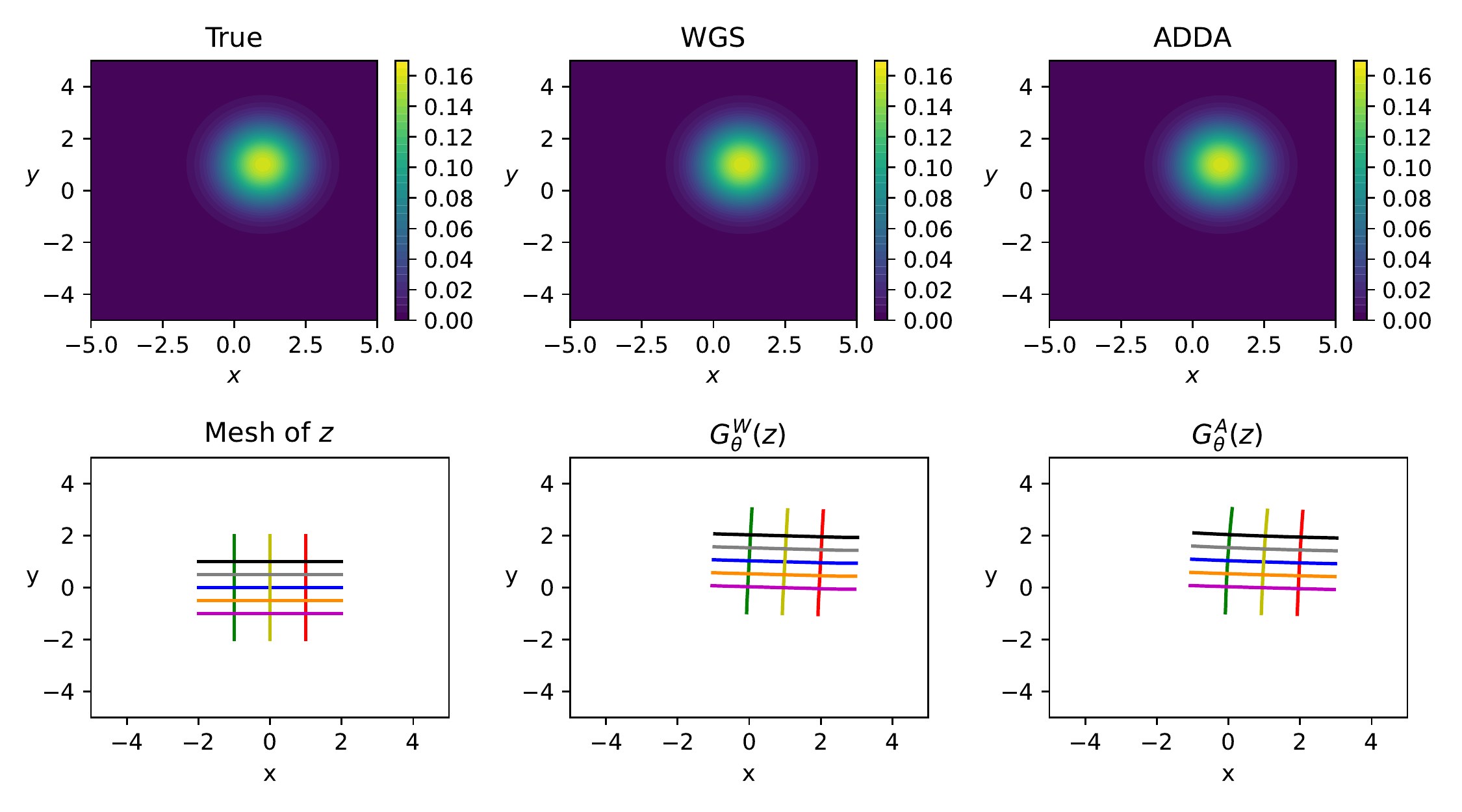}
    \caption{(Example 1) Contour plots of the   true   probability density function  $p(x)$ (upper left), $p_\theta^W(x)$ learned by WGS (upper middle) and $p_\theta^A(x)$ learned by ADDA (upper right). The uniform mesh in the base space (lower left), the mesh transformed by $G_\theta^W$ (lower middle) and the mesh transformed by $G_\theta^A$ (lower right).}
    \label{fig:problem_gaussian}
\end{figure}

Figure \ref{fig:problem_gaussian} presents the comparison between the true probability density function $p$ (upper left), the probability density function $p^W_\theta$ learned by WGS (upper middle), and the probability density function $p^A_\theta$ learned by ADDA (upper right). And the uniform mesh in the base space mapped by $G_\theta^W$ (lower middle) and $G_\theta^A$ (lower right) are almost the same. In Figure \ref{fig:problem_gaussian}, we observe that WGS and ADDA can obtain almost the same as the exact solution.

\begin{figure}[ht]
    \centering
    \includegraphics[width=0.9\textwidth]{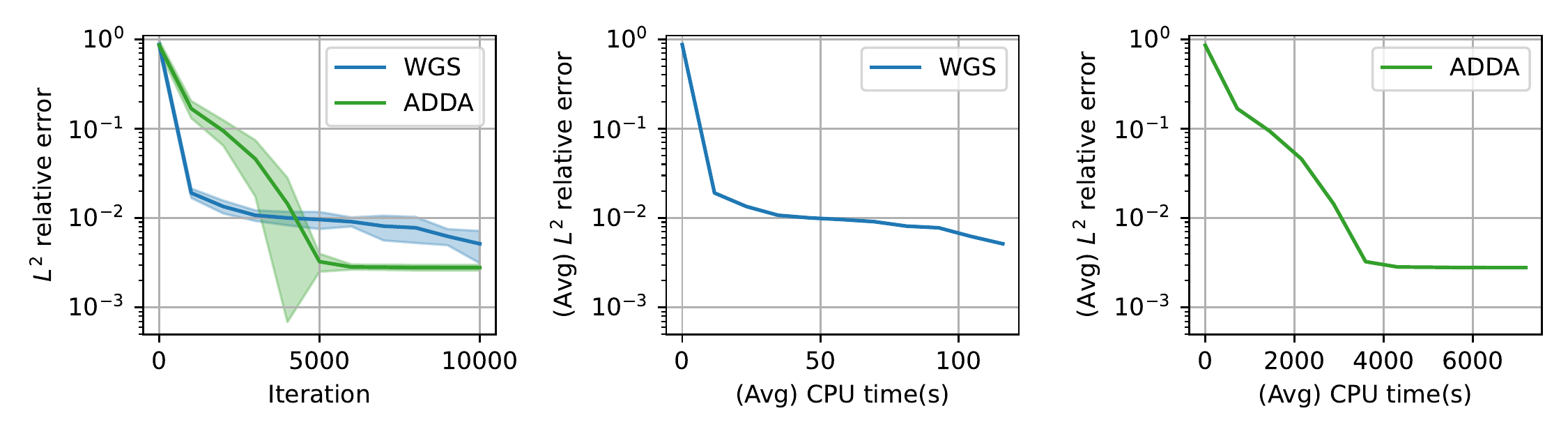}
    \caption{(Example 1) The left panel displays the $L^2$ relative error versus the iteration for the learned solution $p_\theta$ obtained from ten different runs using WGS and ADDA. The middle panel and the right panel depict the average $L^2$ relative error against average CPU time for WGS and ADDA, respectively.}
    \label{fig:problem_gaussian_error}
\end{figure}
In Figure \ref{fig:problem_gaussian_error}, we observed that WGS achieves the  $L^2$ relative error $e_p$ approximately $10^{-2}$ faster than ADDA. However, ADDA can achieve a smaller $e_p$ than WGS throughout the entire iteration. It is worth noting that ADDA requires a significant amount of CPU time than WGS. Since ADDA requires the computation of the Jacobian of the generative map $G_\theta^A$ and the gradient of $p_\theta^A$, WGS only needs the computation of the map $G_\theta^W$. The loss function of WGS can be computed in matrix form.  We conclude that WGS strikes a reasonable balance between computational time and numerical error, yielding satisfactory results,  and that WGS shows a significant improvement in efficiency   for this example.

\subsection{Example 2: A two-dimensional system with two metastable states}
\label{subsec::eg2}

In this section, we consider the following two-dimensional dynamical system \cite{lin2023computing}
\begin{equation}
    \left\{\begin{array}{l}
\dx=\left[\frac{1}{5}x(1-x^2)+y(1+\sin x)\right]\dt+\sqrt{\frac{2}{5}\varepsilon } \dW_1, \\
\dy=\left[-y+2x(1-x^2)(1+\sin x)\right]\dt+\sqrt{2 \varepsilon } \dW_2.
\end{array}\right.
\end{equation}
This system has two metastable states at $x_1=(-1,0)^\top$ and $x_2=(1,0)^\top$ and one unstable stationary point at $x_3=(0,0)^\top$. The invariant measure is thus bi-modal with the two centers near the two metastable states. 
As $\varepsilon $ becomes smaller, the barrier between two metastable states will increase, and it takes exponentially longer physical time for the distribution to reach the invariant measure.  In addition, this is an asymmetric bi-model example, where the mode on the left is more dominant since the probability ratio of the two modes, $\mbox{Prob}(X<0)/\mbox{Prob}(X>0)$, is as large as 4 or 5 in our tests.

In this example, we test WGS for $\epsilon = 0.05$, $0.1$, and $0.2$. For all cases, $N = 10000$ sample points are drawn from the base distribution. 
We use $N_\varphi = 2000$   test functions with a batch size of $N_\varphi^b = 400$, and perform $N_I = 50000$ iterations. The parameter $\kappa$ is gradually decreased from an initial value to a lower value, and then held constant for the final $30000$ iterations. Specifically, for $\epsilon = 0.2$, $\kappa$ is decreased from $0.5$ to $0.25$; for $\epsilon = 0.1$, from $0.45$ to $0.18$; and for $\epsilon = 0.05$, from $0.45$ to $0.10$.  
The setting used for the ADDA algorithm for comparison and  the computing of the true solution $p$ can be found in the Appendix \ref{appendix:c2}.

%\textbf{
\begin{table*}[htbp]
\footnotesize
\caption{Comparison of WGS and ADDA for solving  Example 2}
\label{tab_e2}
\begin{threeparttable}  
\begin{tabularx}{\textwidth}{XXXXX}
\toprule
Methods & $\varepsilon$ & $e_p$  & Time/Iter\tnote{1} & Number of Iters\tnote{2} \\
\midrule
WGS & $0.2$ & $0.0457\pm 0.0119$ & $0.032$ & $2.5\times 10^5$ \\
ADDA & $0.2$ & $0.4011\pm0.0047$ & $1.900$ & $7.5\times 10^4$ \\
\midrule
WGS & $0.1$ & $0.0734\pm 0.0168$ & $0.032$ & $2.5\times 10^5$ \\
ADDA & $0.1$ & $0.4943\pm0.2262$  & $1.900$ & $7.5\times 10^4$ \\
\midrule
WGS & $0.05$ & $0.0784\pm 0.018$  & $0.066$ & $2.5\times 10^5$ \\
ADDA & $0.05$ & $0.8724\pm0.4186$  & $3.167$ & $7.5\times 10^4$ \\
\bottomrule
\end{tabularx}
\begin{tablenotes}
        \footnotesize
        \item[1]{\normalfont The time is the total CPU time during the training of WGS and ADDA. 
        \item[2] The number of Iters  is defined  as $N_I\times  \lceil N_\varphi/N_\varphi^b \rceil$   for WGS   and    $N_I^p\times \lceil N_p/N_p^b\rceil\times N_{\text{adaptive}} $ for ADDA, respectively. Refer to Algorithm \ref{algorithm} and Algorithm \ref{alg:adda} in the Appendix \ref{appendix:c2}  for the meaning of these parameters}.
 
\end{tablenotes}
\end{threeparttable}
\end{table*}
%}

 In Table \ref{tab_e2}, we provide a quantitative  assessment of the numerical solutions  for various $\varepsilon$ values.  
The results are based on the six independent runs with different random seeds,  so the  standard
 deviation of the error $e_p$ is also reported. 
 Compared to the ADDA method, the WGS demonstrates much lower training cost  per iteration as expected. More importantly, the WGS achieves a lower relative error and the ADDA has the larger relative error. The reason is that in the ADDA method, the generative map $G_\theta$  is trapped in one mode  or shrinks into a delta density. This  also  results in a large standard deviation of the relative error for ADDA, particularly when $\varepsilon = 0.05$.

\begin{figure}[htbp]
    \centering
    \includegraphics[width=.8\textwidth]{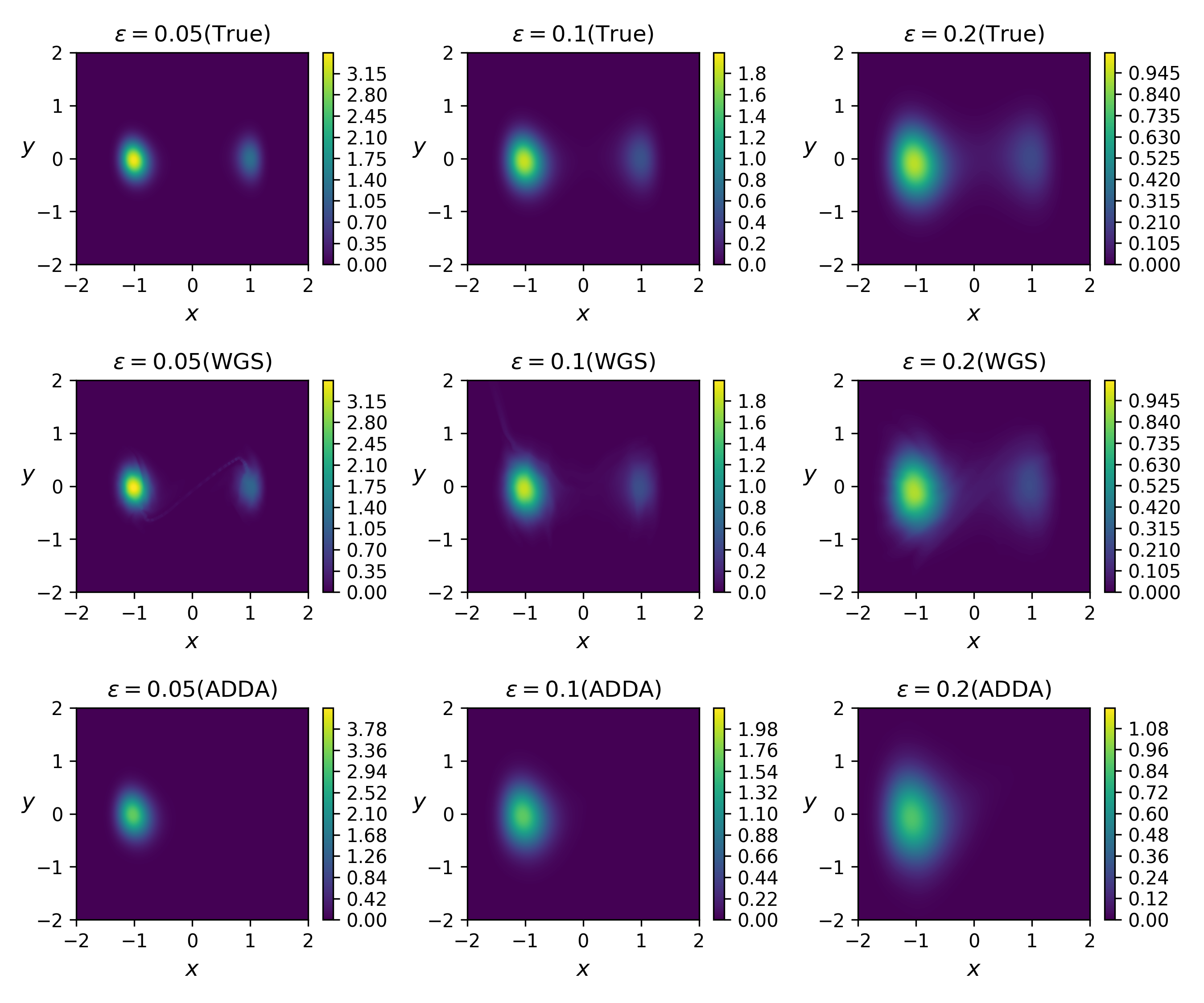}
    \caption{(Example 2) Contour plots of the probability density function of invariant measure $p(x,y)$, where $p(x,y)$ is the finite difference solution of the FP equation (top), $p^W_\theta(x,y)$ learned by WGS (middle) and $p^A_\theta(x,y)$ learned by ADDA for $\varepsilon =0.05$ (left), $\varepsilon =0.1$ (middle) and $\varepsilon =0.2$ (right).}
    \label{fig:problemdw_pdf}
\end{figure}

Figure \ref{fig:problemdw_pdf} compares the true probability density function $p$ (upper panel), the probability density function $p^W_\theta$ learned by WGS (middle panel), and the probability density function $p^A_\theta$ learned by ADDA (lower panel). The  contour plots of the corresponding potentials  are also included in the Appendix \ref{appendix:c1}.  
To illustrate how   the WGS generative map $G_\theta$ finds the two modes during the training,  we show  a typical run  in  Figure \ref{fig:problemdw_error} for $\varepsilon=0.2$ and $\varepsilon=0.05$,  
by plotting the relative errors and the 
generated samples   
during the different training stage. These plots show that WGS  can capture  two metastable states rather quickly within the first 1000 iterations, and then quickly converges to the true distribution. 

To further quantify the capability of capturing two distinct modes, in 
the Appendix \ref{appendix:c2}  we check the   probability that the $x$-component is positive: $\mbox{Pr}(X>0)$, which  is expected to converge to a constant  strictly between zero and one. Figure \ref{fig:e2_prob_x} in Appendix \ref{appendix:c2}
shows  the evolution of this probability $\mbox{Pr}(X>0)$  during the training  for both WGS and ADDA, which  validates the successful performance of  WGS  in identifying both metastable states.
In particular, even when the WGS method captures only one mode initially (at  $\varepsilon=0.05$),  it successfully identifies the other mode as the training progresses. These phenomena are detailed in Appendix \ref{appendix:c2}.

\begin{figure}[htbp]
    \centering
    \includegraphics[width=1.1\textwidth, height=0.66\textwidth]{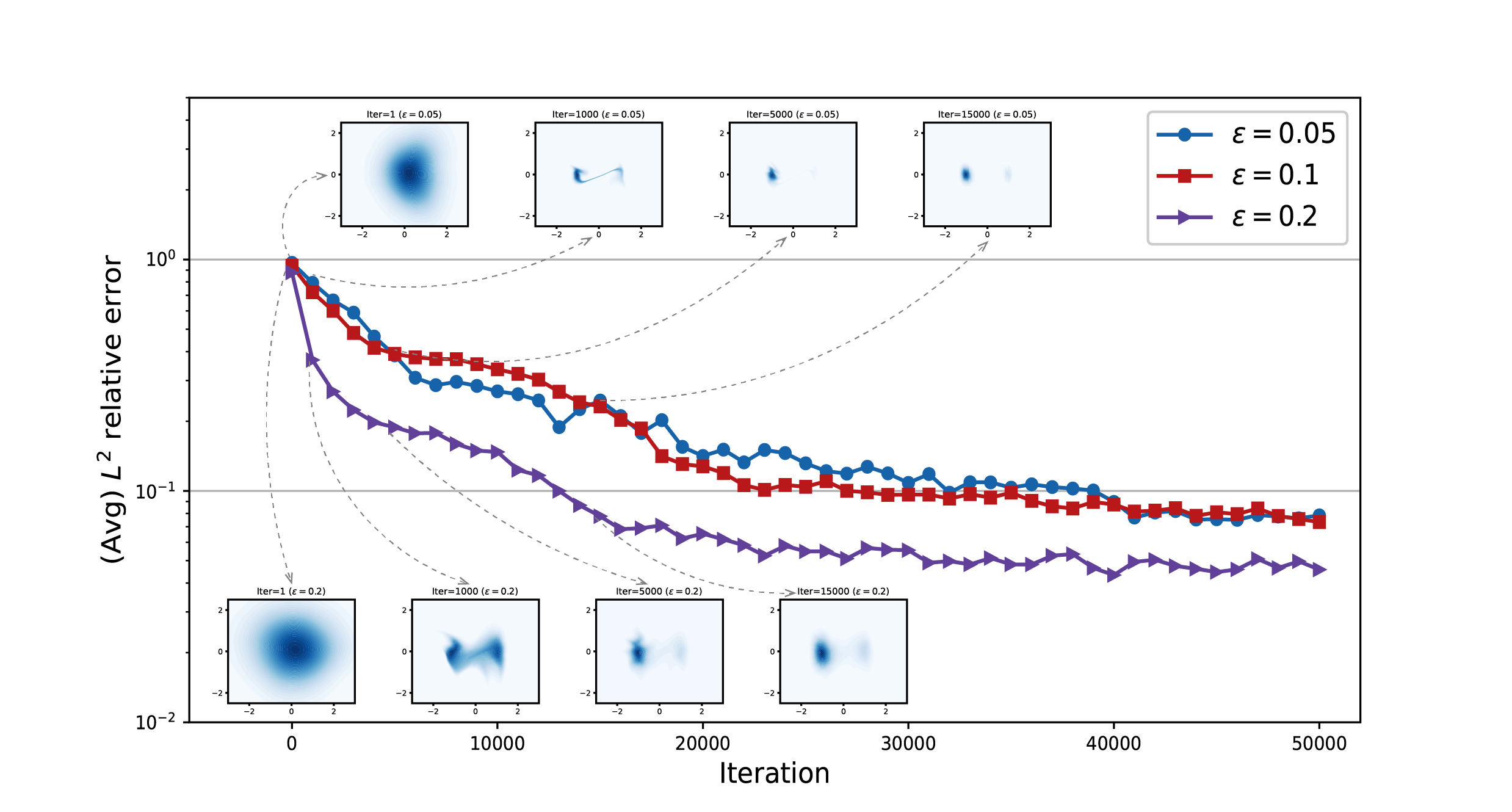}
    \caption{{(Example 2) The average $L^2$ relative error versus the iteration is shown for different values of $\epsilon$. The upper panel displays the contour plot of the probability density function learned by WGS at specific iterations for $\epsilon = 0.05$, while the lower panel shows the corresponding contour plot of the probability density function learned by WGS for $\epsilon = 0.2$.}}
    \label{fig:problemdw_error}
\end{figure}

\subsection{Example 3: Lorenz system}
In this section, we apply the WGS to the Lorenz system, which captures complex atmospheric convection. The Lorenz system exhibits chaotic behaviour, which is sensitive to initial conditions, known as the butterfly effect. The dynamical equations of the Lorenz system in the presence of noise in the three-dimensional space are given by
\begin{equation}
    \left\{\begin{array}{l}
\dx=\beta_1(y-x)\dt+\sqrt{2 \varepsilon } \dW_1, \\
\dy=\left(x\left(\beta_2-z\right)-y\right)\dt+\sqrt{2 \varepsilon } \dW_2, \\
\dz=\left(x y-\beta_3 z\right)\dt+\sqrt{2 \varepsilon } \dW_3,
\end{array}\right.
\end{equation}
where $W=(W_1,W_2,W_3)^\top$ is a three-dimensional white noise, and the diffusion matrix $D=2\varepsilon  \boldsymbol{I}_3$ denotes the three-dimensional identity matrix. We take the parameters $\beta_1=10, \beta_2=28$ and $\beta_3=8/3$. Under these parameter values, the shape of the attractor in the deterministic system resembles a butterfly. We take the parameter $\varepsilon  = 20$.

In this case, we incorporate 12 affine coupling layers within the Real NVP architecture, each consisting of a three-layer neural network. The base distribution is set as $\rho(z)=\mathcal{N}(0,20\boldsymbol{I}_d)$. The training process takes $N_I =7500$ iterations, with the dataset of $N = 10,000$ sample points. We select $N_\varphi = 10,000$ test functions with a batch size of $1000$. We set $\kappa=5$ and the learning rate as $0.0002$ throughout the training phase. 

To validate the accuracy of our method, we use the Euler-Maruyama method to run the SDE to estimate the probability density function $p$ of the invariant measure. In the Euler-Maruyama method, we first sample 1000 points uniformly from $[-25,25]\times[-30,30]\times[-10,60]$. We then simulate 1000 trajectories over a sufficient enough long time $T=10^5$ with time step $\delta t=10^{-3}$. After reaching a burn-in time $T_0=100$, we   save the running data points every 1000 time steps.To estimate the probability density function $p$   using histogram, we   use a fine mesh in  $[-30,30]\times[-40,40]\times[-10,60]$ to compute the frequency of sample data points in each bin. For comparison, we use the learned generative map $G_\theta$ to sample data points and then estimate the probability density function $\hat{p}_\theta$ by the same refined mesh. 

In Figure \ref{fig:lorenz_map}, we choose two random data points and use the black arrow to plot the generative map map $G_\theta$.
In Figure \ref{fig:problemlz_pdf}, we plot the marginal probability density function $p(x,y)$, $p(x,z)$ and $p(y,z)$ estimated by the Monte Carlo method and the learned marginal probability density function $\hat{p}_\theta(x,y)$, $\hat{p}_\theta(x,z)$ and $\hat{p}_\theta(y,z)$. The relative $L^2$ error between $p$ and $\hat{p}_\theta$ is $e_p=0.157$. 

\begin{figure}[ht]
    \centering
    \includegraphics[width=0.9\textwidth]{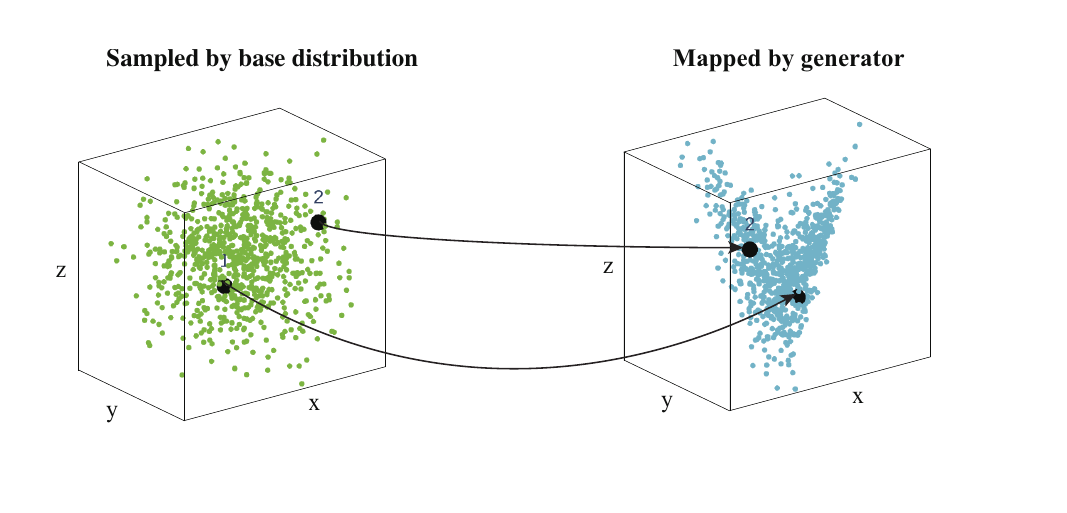}
    \caption{(Example 3) The generative map $G_\theta$ trained by WGS. The left panel shows the data points sampled by the base distribution and the right panel shows the data points mapped by the generative map $G_\theta$. Two pairs of representative data points under this map are highlighted by the black arrows. }
    \label{fig:lorenz_map}
\end{figure}

\begin{figure}[ht]
    \centering
    \includegraphics[width=0.9\textwidth]{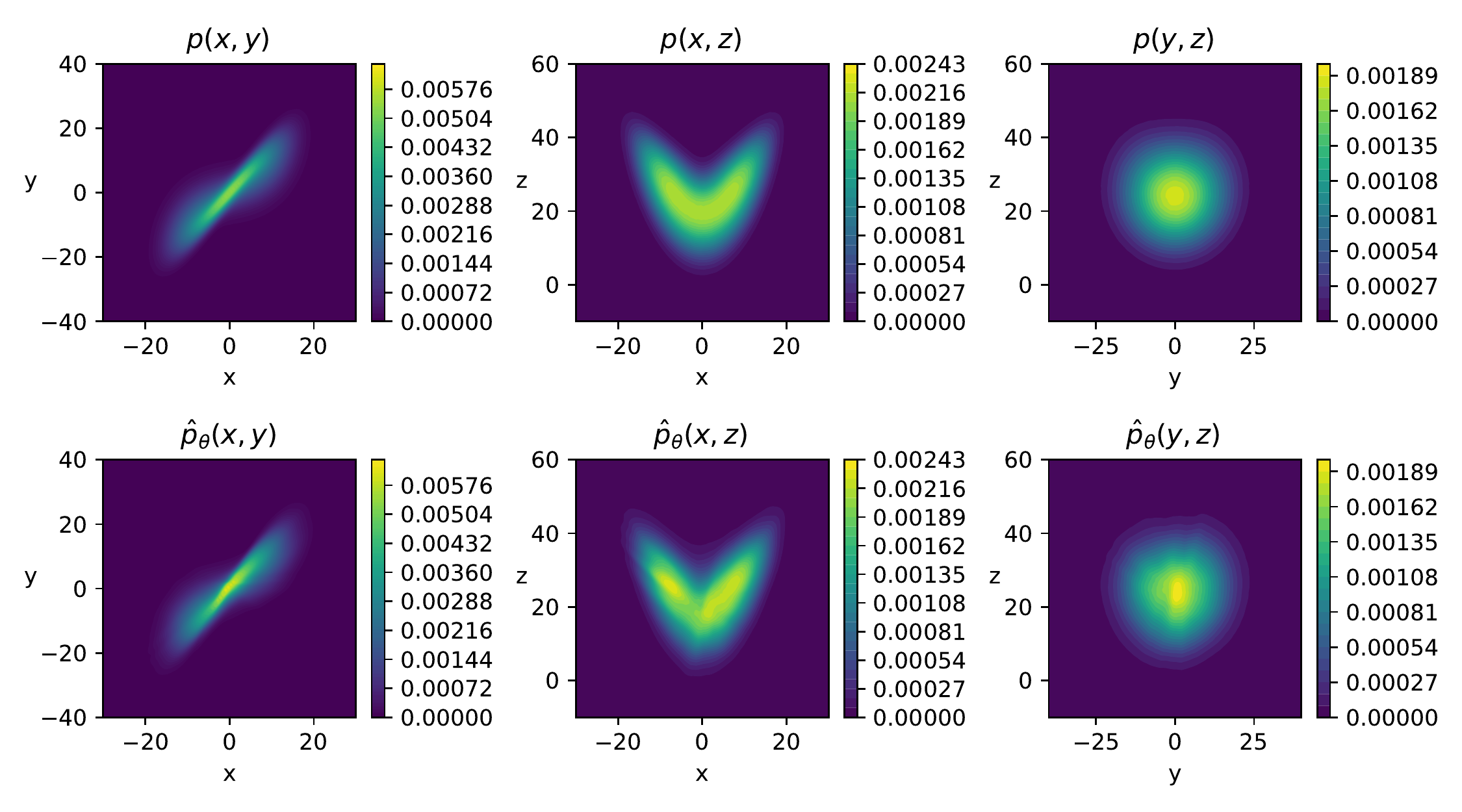}
    \caption{(Example 3) Contour plots of the marginal probability density function $p(x,y)$, $p(x,z)$ and $p(y,z)$ estimated using the Monte Carlo method (upper) and the learned marginal probability density function $\hat{p}_\theta(x,y)$, $\hat{p}_\theta(x,z)$ and $\hat{p}_\theta(y,z)$ computed by WGS. 
    }
    \label{fig:problemlz_pdf}
\end{figure}

\subsection{Example 4: A ten-dimensional problem}
In this example, we test WGS in a ten-dimensional problem to investigate the effect of the scale hyper-parameter in WGS. The system  couples the following five identical and independent two-dimensional systems \cite{lin2022computing,lin2023computing}
\begin{equation}\label{eqn:example4_2d}
    \left\{\begin{array}{l}
\dy_{2 i-1}=(-y_{2 i-1}+y_{2 i}\left(1+\sin y_{2 i-1}\right))\dt+\sqrt{2 \varepsilon} \dW_{2 i-1}, \\
\dy_{2 i}=(-y_{2 i}-y_{2 i-1}\left(1+\sin y_{2 i-1}\right))\dt+\sqrt{2 \varepsilon} \dW_{2 i}, 
\end{array}\right. \quad 1 \leq i \leq 5
\end{equation}
where $W=(W_1,W_2,\cdots,W_{10})^T$ is a $10d$ Brownian motion. By using the transformation of $x=By\in\mathbb{R}^{10}$, where $B\in\mathbb{R}^{10\times 10}$ is a given matrix and $y=(y_1,\cdots,y_{10})$, the dynamics of the variable $x$ in our interest  is governed by the equation
\begin{equation}\label{eqn:example4_10d}
    \d x=f(x)\dt+\sqrt{2\varepsilon}B\d W,
\end{equation}
where the force $f$ can be determined by the transformation $x=By$. The matrix $B=[b_{ij}]$ is given by
$$
b_{ij}=\left\{\begin{array}{cl}
0.8, & \text { for } i=j=2 k-1,1 \leq k \leq 5 \\
1.25, & \text { for } i=j=2 k, 1 \leq k \leq 5 \\
-0.5, & \text { for } j=i+1,1 \leq i \leq 9 \\
0, & \text { otherwise }
\end{array}\right.
$$
and we set $\varepsilon=0.1$.  {By the transformation of $x=By$, one can   show that the invariant distribution of the system \eqref{eqn:example4_10d} is given by}
$
p(x)=\prod_{i=1}^5p_0(y_{2i-1},y_{2i}),
$
where $(y_1,\cdots,y_{10})=B^{-1}x$,  and $p_0$ is the invariant distribution of the $2d$ system \eqref{eqn:example4_2d} and  can be computed by the finite difference method.
\begin{figure}[ht]
    \centering
    \includegraphics[width=0.9\textwidth]{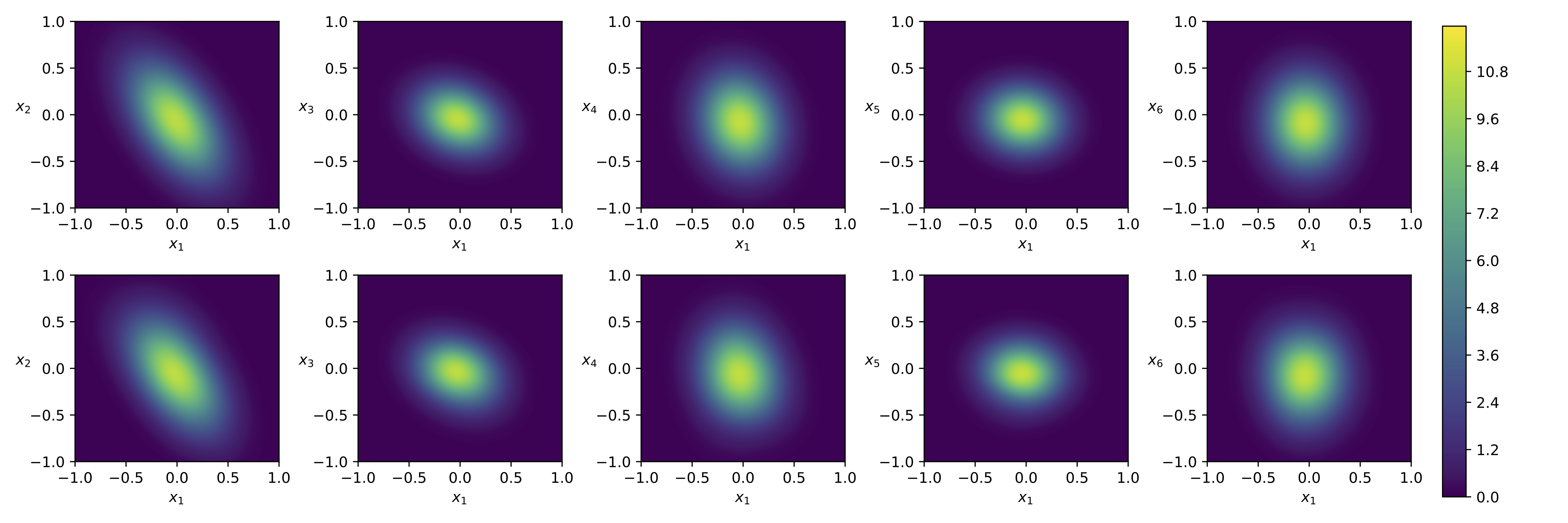}
    \caption{(Example 4) Cross sections of the true solution $p$ (top) and the learned solution $p_\theta$ push forward by $G_\theta$ (bottom), projected on the $x_1$-$x_i$ plane, $2\leq i\leq 6$, where the other
    coordinates are set to zero.}
    \label{fig:problem10dim}
\end{figure}

We incorporate 12 affine coupling layers within the Real NVP architecture, each consisting of a three-layer neural network.  We train the WGS   for    $N_I = 100$ iterations, and the dataset comprising  $N = 30,000$ sample points. We select $N_\varphi = 30,000$ with a batch size of $100$. The learning rate is set to $0.0001$.
We use the scale hyper-parameter $\kappa$   gradually reducing  from $0.8$ to $0.4$. 
For each of the five planes depicted in  Figure \ref{fig:problem10dim}, we compute the relative error by comparing the learned probability density function, denoted as $p_\theta$, with its projection onto the two-dimensional plane while keeping the remaining coordinates fixed at 0. The corresponding relative errors for the five planes are $0.0249$, $0.0265$, $0.0456$, $0.0477$, and $0.0322$, respectively. The results demonstrate a satisfactory agreement between the two solutions across all five cross sections, encompassing both high-probability and low-probability regions within this high-dimensional system.

We remark that using   $\kappa$ that decreases gradually from high to low can perform better. In Figure \ref{fig:problem10dim_parameters}, we compare   the performance when   $\kappa=0.8$ is fixed and when $\kappa$ is gradually reduced  from 0.8 to 0.4 over  five   runs of our algorithm. We observe that different choices of $\kappa$ can influence the convergence behavior  of the method, and gradually reducing $\kappa$ can perform better than using  a fixed value of $\kappa$ during the algorithm's training. 
\begin{figure}[ht]
    \centering
   \includegraphics[width=0.9\textwidth]{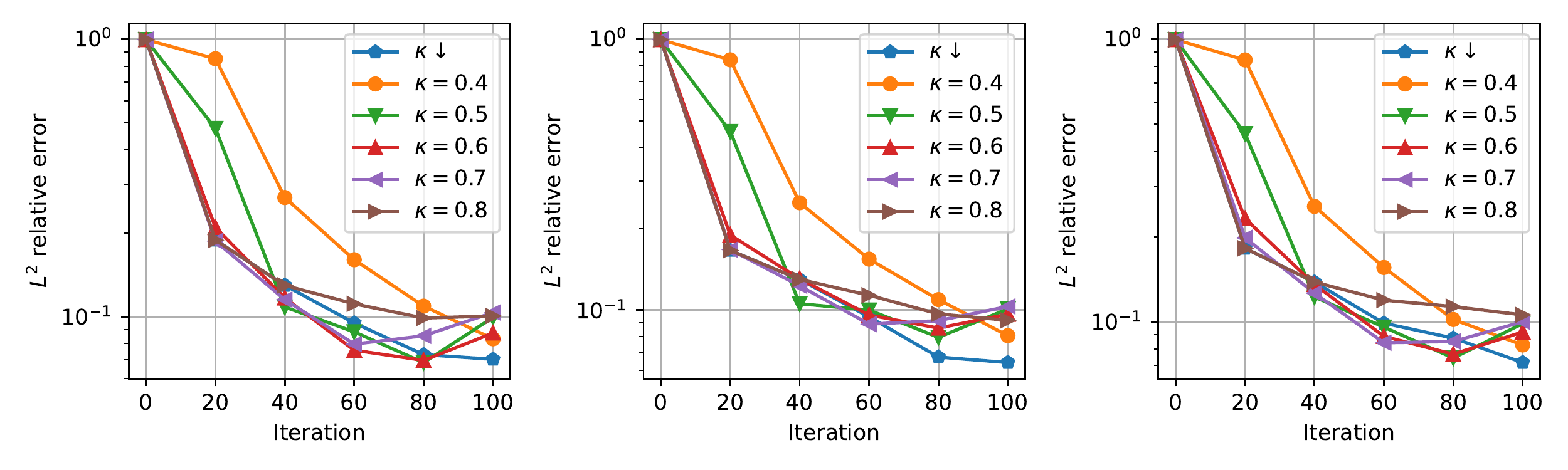}
    \caption{ (Example 4) Comparison with different choice of $\kappa$ by the relative error for the cross-section of the learned solution $p_\theta$ on the $x_1$-$x_2$ plane (left), $x_1$-$x_4$ plane (middle), $x_1$-$x_6$ plane (right), where the other coordinates are set to zero. $\kappa\downarrow$ denotes that $\kappa$ gradually decrease from $0.8$ to $0.4$.}
    \label{fig:problem10dim_parameters}
\end{figure}

\subsection{Example 5: A high-dimensional problem}
In this section, we test the WGS   in a  high-dimensional problem as follows
\begin{equation}\label{eqn:example5}
    \dx_i = -i(x_i-1)\dt+\sqrt{\frac{2}{i}}\dW_i,\quad i\in\{1,2,\cdots,d\}.
\end{equation}
The invariant distribution of   \eqref{eqn:example5} is $p(x)=\mathcal{N}(\mu,\Sigma)$, where $\mu=(1,\ldots, 1)$ 
  and $\Sigma$ is a $d$-dimensional diagonal matrix with the $i$-th diagonal entry defined as $1/i^2$ for $i = 1, 2, \dots, d$. As the dimension $i$ increases, the mean is the constant one, but the corresponding variance shrinks.

The base distribution is the standard Gaussian distribution in $\mathbb{R}^d$.
For the scale hyper-parameter $\kappa$ in the Gaussian test function, we use a mixed group of test for $\kappa$. The first group use a fixed  $\kappa$,  the second group  follows an exponential decay schedule and 
the third group is the random $\kappa$ uniformly from an interval. 
This hybrid strategy  for $\kappa$ can 
 not only help the robustness in the early training period but also  improve  the accuracy in later period. The learning rate followed an exponential decay schedule. The detailed settings of these hyper-parameters can be found  in the Appendix \ref{appendix:d}.

In Figure \ref{fig:40dim} and Figure \ref{fig:100dim}, we show  the estimated and true means and variances in each dimension for $d=40$ and $d=100$, respectively. The numerical means and variances in each dimension are estimated from the samples produced by the trained generative map $G_\theta$. 

Based on the true invariant distribution $p(x)=\mathcal{N}(\mu,\Sigma)$, the relative errors for the numerical mean $\tilde{\mu}$ and the variance matrix $\tilde{\Sigma}$, $e_M:=\frac{\|\tilde{\mu}-\mu\|_2}{\|\mu\|_2}$ an $e_C:=\frac{\|\tilde{\Sigma}-\Sigma\|_F}{\|\Sigma\|_F} $
are measured, respectively. 
For $d=40$, $e_M= 2.90\times 10^{-3}$ and $e_C= 3.25\times 10^{-2}$;
and for $d=100$, the relative errors   are $e_M=1.03\times 10^{-3}$ and $e_C=8.38\times 10^{-2}$.  The decay of the WGS loss and these two errors are plot in the Appendix \ref{appendix:d} (Figure \ref{fig:40100dim_loss}).
These errors show the scalability of the WGS in handling higher dimensional problem than the previous examples.

\begin{figure}[ht]
    \centering
    \includegraphics[width=0.9\textwidth]{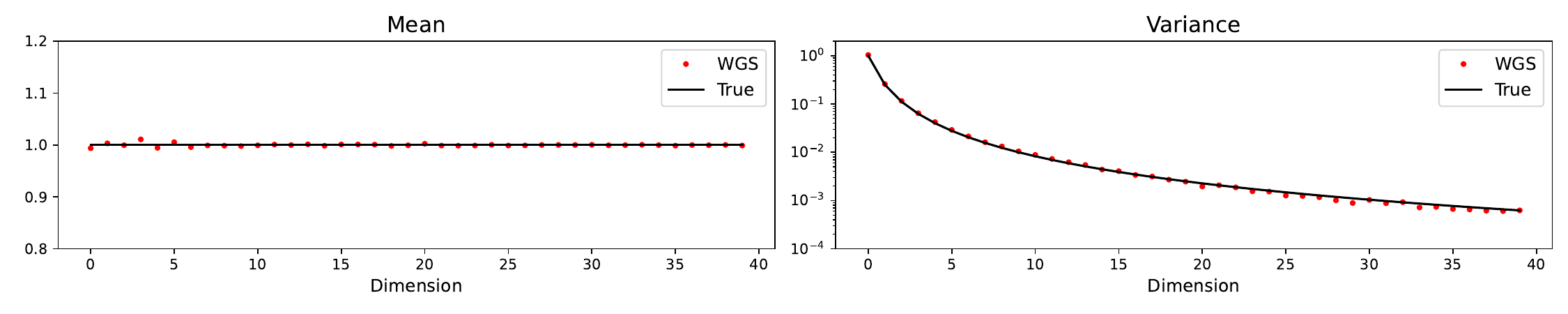}
    \caption{(Example 5-$40$ dim) The estimated means and variances (red points) with the true means and variances (black curves) in each dimension.}
    \label{fig:40dim}
\end{figure}

\begin{figure}[ht]
    \centering
    \includegraphics[width=0.9\textwidth]{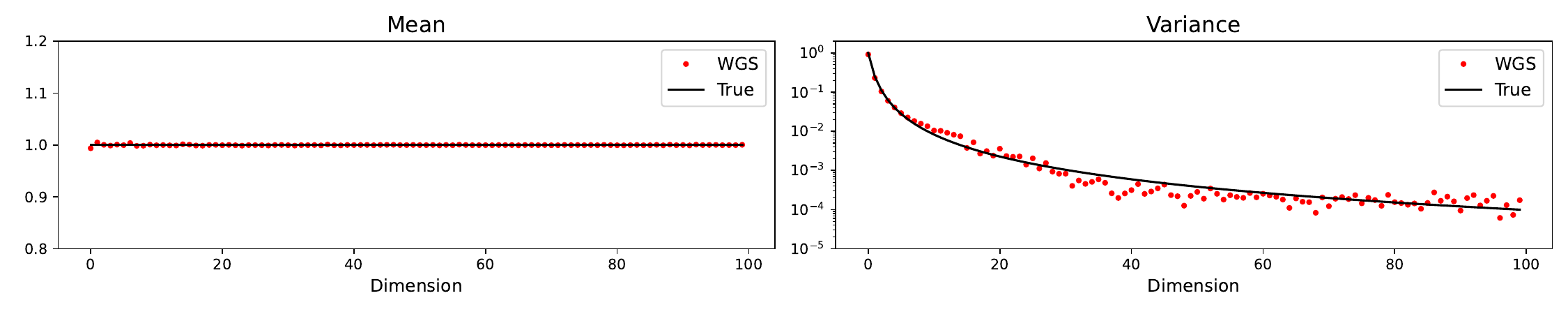}
    \caption{(Example 5-100dim) The estimated means and variances (red points) with the true means and variances (black curves) in each dimension.}
    \label{fig:100dim}
\end{figure}

\section{Conclusion and Outlook}\label{s5}
In this paper, we have presented a novel method called weak generative sampler (WGS) for sampling the invariant measure of diffusion processes. The WGS method utilizes the weak form of the stationary Fokker-Planck equation, eliminating the need for Jacobian computations and resulting in significant computational savings. By selecting test functions based on data-driven approaches, the WGS effectively identifies all metastable states in the system. 
The current randomization strategy of the 
Gaussian kernel test function may suggest a new type of 
adaptive strategy  beyond the current mainstream 
adaptive methods of PINN, which will be explored as a future work.
When the equation contains certain parameters, for example  the   noise intensity $\varepsilon$,  the current WGS can be naturally generalized to train a  parametric generative map $G_\theta(z,\varepsilon)$ from $\mathbb{R}^{d+1}$ to $\mathbb{R}^d$ by wrapping  our randomized weak loss function with one more average over the sampled values of the parameters $\varepsilon$.  
We anticipate further exploration of WGS's applications to time-dependent Fokker--Planck equations and McKean-Vlasov problems, expanding its potential for more general scenarios.

\section*{Appendix}

\appendix

\section{Construction of probability measures on test function spaces}\label{appendix:a}
We show in this appendix the existence of the probability measures required in Assumption \ref{assum:bounded} and \ref{assumptionptheta}, by constructing such a probability measure on the different test function spaces.  We emphasize that the construction here only to show the existence in a theoretical perspective, not the practical data-driven construction of the test function in our Algorithm \ref{algorithm}, which certainly satisfies Assumption \ref{assum:bounded} and \ref{assumptionptheta}.

We first consider the Sobolev space $\left(H^2(\mathbb{R}^d),\langle\cdot,\cdot\rangle\right)$, which is a separable Hilbert space, where the inner product is given by
\begin{equation*}
    \langle f,g\rangle=\sum_{|\alpha|\leq 2}\int_{\mathbb{R}^d}(\partial^\alpha f)(\partial^\alpha g)\d x.
\end{equation*}

We denote by $L(H^2(\mathbb{R}^d))$ the Banach algebra of all continuous linear operators from $H^2(\mathbb{R}^d)$ to $H^2(\mathbb{R}^d)$, by $L^+(H^2(\mathbb{R}^d))$ the set of all $T\in L(H^2(\mathbb{R}^d))$ which are symmetric $(\langle Tf,g\rangle = \langle f,Tg\rangle,\ f,g\in H^2(\mathbb{R}^d))$, and by $L_1^+(H^2(\mathbb{R}^d))$ the set of all operators $Q\in L^+(H^2(\mathbb{R}^d))$ of trace class that is such that $\mathrm{Tr} Q:=\sum_{k=1}^\infty\langle Qe_k,e_k\rangle<\infty $ for one complete orthonormal system $(e_k)$ in $H^2(\mathbb{R}^d)$.  We know that \cite[Section 1.5 \& Chapter 9]{da2006introduction}, there exist non-degenerate (i.e., $\mathrm{Ker}(Q)=\{f\in H^2(\mathbb{R}^d):\ Qf=0\}=\{0\}$) Gaussian measures on $\left(H^2(\mathbb{R}^d),\|\cdot\|_{H^2(\mathbb{R}^d)}\right)$ where the associated norm $\|\cdot\|_{H^2(\mathbb{R}^d)}$ is 
\begin{equation*}
    \|f\|_{H^2(\mathbb{R}^d)}=\sum_{|\alpha|\leq2}\left(\int_{\mathbb{R}^d}|\partial^\alpha f|^2\d x\right)^{\frac{1}{2}}.
\end{equation*}
Let $\mathcal{N}_{0,Q^*}$ be such a Gaussian measure with mean 0 and covariance $Q^*\in L_1^+(H^2(\mathbb{R}^d))$. There exists a sequence of non-negative numbers $(\lambda_k)$ such that
\begin{equation*}
    Q^*e_k=\lambda_k e_k,\quad k\in\mathbb{N}.
\end{equation*}

For any $f\in H^2(\mathbb{R}^d)$ we set $f_k=\langle f,e_k\rangle$, $k\in\mathbb{N}$. Now let us consider the natural isomorphism $\Gamma$ between $H^2(\mathbb{R}^d)$ and the Hilbert space $l^2$ of all sequence $(f_k)$ of real numbers such that
\begin{equation*}
    \sum_{k=1}^\infty|f_k|^2<\infty,
\end{equation*}
defined by 
\begin{equation*}
    H^2(\mathbb{R}^d)\to l^2,\quad f\mapsto \Gamma(f)=(f_k).
\end{equation*}
And we shall identify $H^2(\mathbb{R}^d)$ with $l^2$ and thus the corresponding probability measure for $\mathcal{N}_{0,Q^*}$ is the following product measure
\begin{equation}\label{prodmeasure}
    \mu:=\bigtimes_{k=1}^\infty\mathcal{N}_{0,\lambda_k},
\end{equation}
where $\mathcal{N}_{0,\lambda_k}$ is the Gaussian measure in $\mathbb{R}$ with mean 0 and variance $\lambda_k$. Though $\mu$ is defined on $\mathbb{R}^\infty:=\bigtimes_{k=1}^\infty\mathbb{R}$, it is concentrated in $l^2$ \cite[Proposition 1.11]{da2006introduction}. However, every bounded Borel set, e.g. unit ball, in $\mathcal{B}(H^2(\mathbb{R}^d))$ under such Gaussian measure $\mu$ has zero mass. Now we turn to its projection. For any given $n\in\mathbb{Z}$, we consider the projection mapping $P_n:\ H^2(\mathbb{R}^d)\to P_n(H^2(\mathbb{R}^d))$ defined as
\begin{equation*}
    P_nf=\sum_{k=1}^n\langle f,e_k\rangle e_k,\quad f\in H^2(\mathbb{R}^d).
\end{equation*}
Obviously we have $\lim_{n\to\infty}P_n f=f$ for all $f\in H^2(\mathbb{R}^d)$.

Since the test function space $C_c^\infty(\mathbb{R}^d)$ is dense in $H^2(\mathbb{R}^d)$, so $\mathcal{B}(H^2(\mathbb{R}^d))\cap C_c^\infty(\mathbb{R}^d)$ is a Borel $\sigma$-field on $C_c^\infty(\mathbb{R}^d)$, that is for every $A\in C_c^\infty(\mathbb{R}^d)$ there exists $\tilde{A}\in \mathcal{B}(H^2(\mathbb{R}^d))$ such that $A=\tilde{A}\cap C_c^\infty(\mathbb{R}^d)$. For a fixed $n\in\mathbb{Z}$, define the measure $\mathbb{P}$ as
\begin{equation*}
    \mathbb{P}(A):= \left(\bigtimes_{k=1}^n\mathcal{N}_{0,\lambda_k}\right)(P_n(\tilde{A})),\quad \forall A\in \mathcal{B}(H^2(\mathbb{R}^d))\cap C_c^\infty(\mathbb{R}^d).
\end{equation*}
It is easy to see $\mathbb{P}$ is a probability measure on $\left(C_c^\infty(\mathbb{R}^d),\mathcal{B}(H^2(\mathbb{R}^d))\cap C_c^\infty(\mathbb{R}^d)\right)$ satisfying the Assumption \ref{assumptionptheta}-\ref{subassum:L2R} due to the following property.

\begin{proposition}\label{probabilityN}
    For any given $n\in\mathbb{Z}$, $r>0$ and $ r_0  >0$, we have
    \begin{equation*}
        \inf_{f\in \bar{B}(0,r)}\left(\bigtimes_{k=1}^n\mathcal{N}_{0,\lambda_k}\right)(P_n(\bar{B}_{H^2(\mathbb{R}^d)}(f, r_0  )))>0,
    \end{equation*}
    where $\bar{B}_{H^2(\mathbb{R}^d)}(f, r_0  )=\{g\in H^2(\mathbb{R}^d):\ \|f-g\|_{H^2(\mathbb{R}^d)}\leq r_0  \}$ is the closed ball centered at $f$ and with $ r_0  $ radium in $H^2(\mathbb{R}^d)$.
\end{proposition}
\begin{proof}
    Note that, $\mathcal{N}_{0,Q^*}$ is a nondegenerate Gaussian measure on $H^2(\mathbb{R}^d)$, and 
    \begin{equation*}
        \begin{aligned}
            \left(\bigtimes_{k=1}^n\mathcal{N}_{0,\lambda_k}\right)(P_n(\bar{B}(f, r_0  ))=&\ \left(\bigtimes_{k=1}^n\mathcal{N}_{0,\lambda_k}\right)\left(P_n(\left\{g\in\mathbb{R}^\infty:\ |f-g|_{l^2}^2\leq r_0   \right\}) \right)\\
            =&\ \left(\bigtimes_{k=1}^n\mathcal{N}_{0,\lambda_k}\right)\left(P_n\left\{g\in\mathbb{R}^\infty:\ \sum_{k=1}^\infty|(f-g)_k|^2\leq  r_0   \right\} \right)\\
            =&\ \left(\bigtimes_{k=1}^n\mathcal{N}_{0,\lambda_k}\right)\left\{g\in\mathbb{R}^n:\ f_k-\sqrt{ r_0  }\leq g_k\leq f_k + \sqrt{ r_0  } \right\}\\
            =&\ \prod_{k=1}^n \mathcal{N}_{0,\lambda_k}\left( \left[f_k-\sqrt{ r_0  },f_k + \sqrt{ r_0  }\right]\right) > 0.
        \end{aligned}
    \end{equation*}

Also note that the 1-dimensional Gaussian distribution is continuous, and thus the measure $\left(\bigtimes_{k=1}^n\mathcal{N}_{0,\lambda_k}\right)(P_n(\bar{B}(\cdot, r_0  )))$ is a continuous function on $\left(H^2(\mathbb{R}^d), \|\cdot\|_{H^2(\mathbb{R}^d)}\right)$ and we have $\left(\bigtimes_{k=1}^n\mathcal{N}_{0,\lambda_k}\right)(P_n(\bar{B}(\cdot, r_0  ))):\ H^2(\mathbb{R}^d)\to[0,1]$.

Thus the infimum value of $\left(\bigtimes_{k=1}^n\mathcal{N}_{0,\lambda_k}\right)(P_n(\bar{B}(\cdot, r_0  )))$ on any closed ball can be achieved and is nonnegative. In particular,
\begin{equation*}
        \inf_{f\in \bar{B}(0,r)}\left(\bigtimes_{k=1}^n\mathcal{N}_{0,\lambda_k}\right)(P_n(\bar{B}(f, r_0  )))\geq \prod_{k=1}^n \mathcal{N}_{0,\lambda_k}\left( \left[r-\sqrt{ r_0  },r + \sqrt{ r_0  }\right]\right) > 0,
    \end{equation*}
\end{proof}

Note that $L^2(\mathbb{R}^d)$ is also a separable Hilbert space and its elements can also be approximated by sequences of $C^\infty_c(\mathbb{R}^d)$, and thus in a similar way as before we can find a non-degenerate Gaussian measure on $L^2(\mathbb{R}^d)$ and construct a probability measure $\mathbb{P}$ based on it and its projection on space $C_c^\infty(\mathbb{R}^d)$.

Next, consider the construction of $\mathbb{P}$  for test functions restricted in $B_{r_\varepsilon}$. As similar with before, the Sobolev space $H^2(B_{r_\varepsilon})$ is a separable Hilbert space and thus there exist Gaussian measures on $(H^2(B_{r_\varepsilon}),\mathcal{B}(H^2(B_{r_\varepsilon})),\|\cdot\|_{H^2(B_{r_\varepsilon})})$. We choose one of any non-degenerate Gaussian measures and denote it by $\mathcal{N}$. We know that every function in $H^2(B_{r_\varepsilon})$ can be approximated by a sequence of functions of $C^\infty(\bar{B}_{r_\varepsilon})$. Let $C^{\infty*}(B_{r_\varepsilon}):=\{\varphi|_{B_{r_\varepsilon}}:\ \varphi\in C^{\infty}(\bar{B}_{r_\varepsilon})\}$, then $C^{\infty*}(B_{r_\varepsilon})$ is dense in $H^2(B_{r_\varepsilon})$. Equip the space $C^{\infty*}(B_{r_\varepsilon})$ with the Borel $\sigma$-field $\mathcal{B}(H^2(B_{r_\varepsilon}))\cap C^{\infty*}(B_{r_\varepsilon})$. For every $A\in \mathcal{B}(H^2(B_{r_\varepsilon}))\cap C^{\infty*}(B_{r_\varepsilon})$ there exists $\tilde{A}\in\mathcal{B}(H^2(B_{r_\varepsilon}))$ such that the closure of $A$ is identical to $\tilde{A}$. 

Define a measure $\mathbb{P}_0$ on the $\sigma$-field $\mathcal{B}(H^2(B_{r_\varepsilon}))\cap C^{\infty*}(B_{r_\varepsilon})$ by
\begin{equation*}
    \mathbb{P}_0(A):=\mathcal{N}(\tilde{A}),\quad \forall A\in \mathcal{B}(H^2(B_{r_\varepsilon}))\cap C^{\infty*}(B_{r_\varepsilon}).
\end{equation*}
It is easy to see $\mathbb{P}_0$ is a probability measure on $\left(C^{\infty*}(B_{r_\varepsilon}), \mathcal{B}(H^2(B_{r_\varepsilon}))\cap C^{\infty*}(B_{r_\varepsilon})\right)$.

By the Whitney extension theorem, \cite[Theorem I]{whitney1992analytic}, we know that for every function in $C^\infty(\bar{B}_{r_\varepsilon})$ there exists an extension of it in $C^\infty(\mathbb{R}^d)$. So we can use the extension to construct an element belonging to $C^\infty_c(\mathbb{R}^n)$ by truncating the extension on a bigger domain and operating it with a mollifier as before. In converse every element in $C_c^\infty(\mathbb{R}^n)$ restricted on $\bar{B}_{r_\varepsilon}$ belongs to $C^\infty(\bar{B}_{r_\varepsilon})$. And we also know that every element of $C^\infty(\bar{B}_{r_\varepsilon})$ restricted on $B_{r_\varepsilon}$ belongs to $H^2(B_{r_\varepsilon})$.

Now we define an equivalence relation $``\sim"$ in $C_c^\infty(\mathbb{R}^d)$ as follows: $\varphi\sim\psi$ if $\varphi|_{B_{r_\varepsilon}}=\psi|_{B_{r_\varepsilon}}$, i.e., $\varphi(x)=\psi(x)$ when $x\in B_{r_\varepsilon}$. Thus the test function space $C_c^\infty(\mathbb{R}^d)$ under relation $\sim$ and the $H^2(B_{r_\varepsilon})$ norm defines a topological space having the same structure of the topological space $\left(C^{\infty*}(B_{r_\varepsilon}), \mathcal{B}(H^2(B_{r_\varepsilon}))\cap C^{\infty*}(B_{r_\varepsilon})\right)$. Denote such topological space as $(C_c^\infty(\mathbb{R}^d),\mathcal{B}(C_c^\infty(\mathbb{R}^d),\sim,\|\cdot\|_{H^2(B_{r_\varepsilon})}))$. For any $A'\in \mathcal{B}(C_c^\infty(\mathbb{R}^d),\sim,\|\cdot\|_{H^2(B_{r_\varepsilon})})$ there exists $A\in \mathcal{B}(H^2(B_{r_\varepsilon}))\cap C^{\infty*}(\mathbb{R}^d)$ such that
\begin{equation*}
    A'=\{\varphi\in C_c^\infty(\mathbb{R}^d):\ \varphi|_{B_{r_\varepsilon}}\in A\},
\end{equation*}
also in converse way, for any $A\in \mathcal{B}(H^2(B_{r_\varepsilon}))\cap C^{\infty*}(\mathbb{R}^d)$  there exists $A'\in \mathcal{B}(C_c^\infty(\mathbb{R}^d) ,$ $ \sim, \|\cdot\|_{H^2(B_{r_\varepsilon})})$ such that
\begin{equation*}
    A=\{\varphi|_{B_{r_\varepsilon}}\in C^{\infty*}(B_{r_\varepsilon}):\ \varphi\in A'\}.
\end{equation*}
Define a measure $\mathbb{P}$ on the space $C_c^\infty(\mathbb{R}^d)$ as follows,
\begin{equation}
    \mathbb{P}(A'):=\mathbb{P}_0(A),\quad \forall A'\in \mathcal{B}(C_c^\infty(\mathbb{R}^d),\sim,\|\cdot\|_{H^2(B_{r_\varepsilon})}).
\end{equation}
Then we obtain the measure satisfying the assumptions we need. 

Finally, we can conclude that the probability measure $\mathbb{P}$ on $C_0^\infty(U)$, as discussed in Section \ref{ss:bd}, can be constructed in a similar manner. We consider the space $H^2(V)$, where $V$ is a strict subset of $U$. Then the rest of the construction follows in the same way as before.

\section{Comparison of loss landscapes
of the WGS and the PINN}\label{appendix:b}
To illustrate and help understand why the WGS and ADDA show different numerical behaviors, we shall 
compare the weak loss used in WGS and the PINN loss used in ADDA by studying their loss landscapes on a simple example of one dimensional Gaussian mixture probability.

  Suppose the target distribution $p^*$ is the following one dimensional  Gaussian mixture centered at $\pm \mu^*$   with the weight $w^*$:
  \begin{equation} \label{eq:testp}
      p^*(x)=  \frac{1}{\sqrt{2\pi }} \left( w^*e^{-(x-\mu^*)^2/2}+ (1-w^*) e^{-(x+\mu^*)^2/2 } \right)
  \end{equation}  
  where $\mu^*=2$ 
  and $w^*=0.5$ (bi-mode) or $w^*=1$ (single mode).
 $p^*$ is the invariant measure of the  over-damped Langevin  SDE $\d X_t = b(X_t)\d t + \sqrt{2}\d W$ with   the drift $b(x)=\nabla \log p^*(x)$.
We consider an idealistic situation where 
the numerical solution $p_\theta$ is   a family of the following parametric Gaussian mixture:
 $$ p_\theta(x)=  \frac{1}{\sqrt{2\pi (1+\theta_\sigma) }} \left( (w^*+\theta_w)e^{-\frac{(x-\mu^*-\theta_\mu)^2}{2(1+\theta_\sigma)}}+  (1-w^*-\theta_w)e^{-\frac{(x+\mu^*+\theta_\mu)^2}{2(1+\theta_\sigma)} } \right)$$
 where $\theta=(\theta_w, \theta_\mu, \theta_\sigma)$. The optimal parameter $\theta^*=(0,0,0)$ since $p^*=p_{\theta^*}$.

The ADDA uses the ($p_\theta$-weighted) PINN loss 
\begin{equation} \label{eq:ADDAloss}
L_{PINN}(\theta) = \mathbb{E}_{X\sim p_{\theta}}\left[\mathcal{L}p_\theta(X)\right]^2.  
\end{equation}
For the WGS loss, we make a further simplification with only  {\it one}  test function (i.e., $N_\varphi=1$). This test function takes the form of a Gaussian density function   $\varphi=\mathcal{N}(\alpha,\kappa^2)$, where $\alpha$ and $\kappa$ are the hyper-parameter. The WGS loss then is
\begin{equation}
    \label{eq:WGSloss}
    L_{WGS}(\theta; \alpha, \kappa)=(\mathbb{E}_{X\sim p_\theta}(\mathcal{L}^*\varphi(X)))^2
\end{equation}

We shall compare the above two loss functions \eqref{eq:ADDAloss} and \eqref{eq:WGSloss}
by using a sufficient large number of samples to calculate  the expectation $\mathbb{E}_{X\sim p_\theta}$. To visualize the loss landscapes, among three parameters in  $\theta=(\theta_w, \theta_\mu, \theta_\sigma)$, we vary only one of them by fixing the other two at the optimal zero values. So every plot of the loss landscape below is a one-dimensional function. 
Note that while the PINN loss only depends on $\theta$, the  WGS loss landscape also depends on the test function via the hyper-parameters $\alpha$ and $\kappa$.

\subsection{ \texorpdfstring{$p^*$}{} is uni-modal(\texorpdfstring{$w^*=1$}{})}

\begin{figure}[htbp]
    \centering
\includegraphics[width=.75\textwidth]{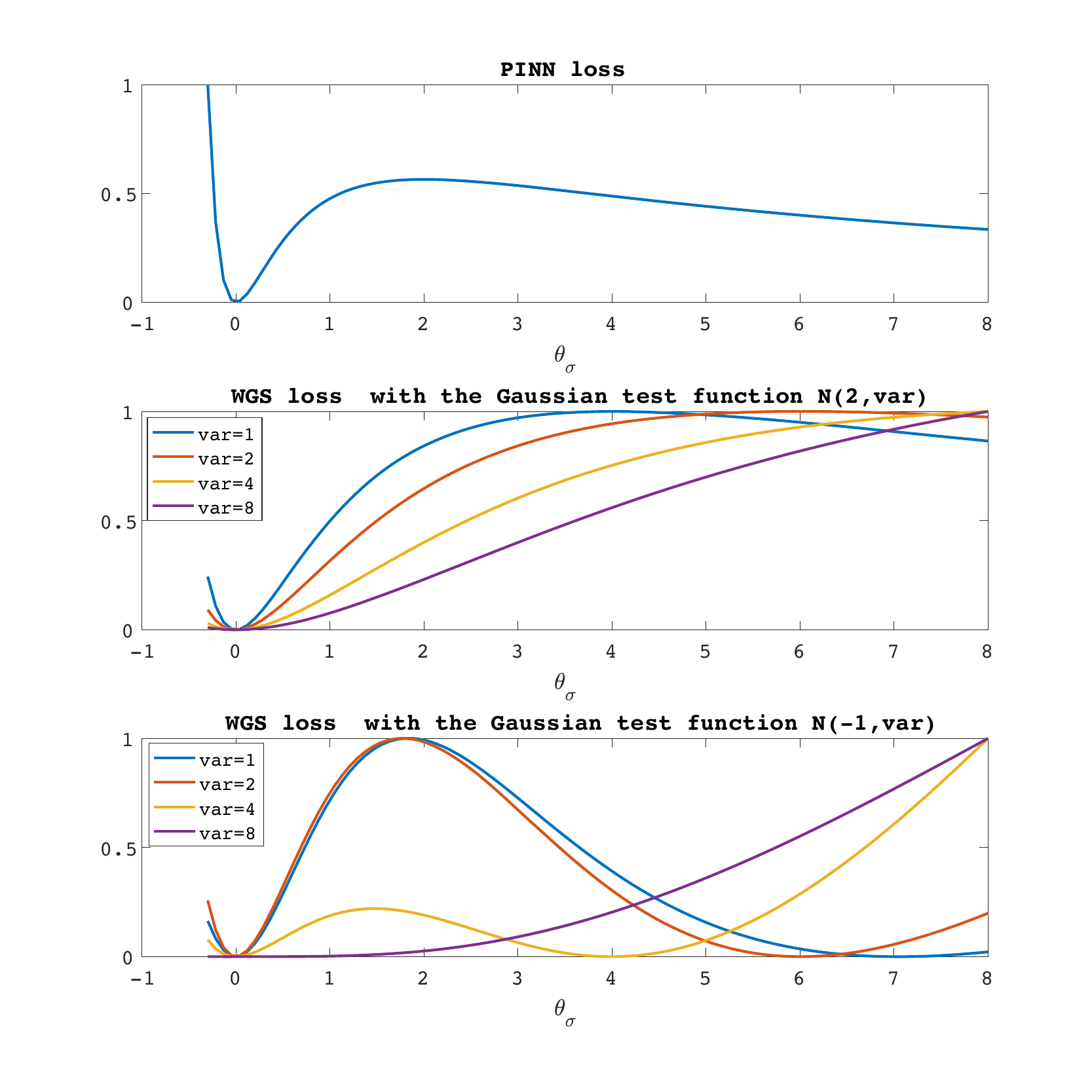}
    \caption{The loss functions of the PINN and the WGS associated vs. the parameter $\theta_\sigma$,
    for the test function with the two centers located as $2$ (the middle panel) and $-1$ (the bottom panel)  respectively, and the four increasing values of $\kappa^2=1,2,4,8$   (``var" in the legend).  }
    \label{fig:test-sigma-one-mode2}
\end{figure}

In this case of $p^*$,
 when we compare the two losses by varying only the weight $\theta_w$ or the mean $\theta_\mu$, both loss functions are convex and the only minimizer is   the true optimal value $\theta^*=0$, which means that all initials under gradient descent can converge to $\theta^*$. This is a good case so we do not show the plot.
 
But when we vary the parameter $\theta_\sigma$ related to the variance 
   (with the other two $\theta_w=\theta_\mu=0$ fixed), we observe distinctive patterns.  The first subplot in   Figure \ref{fig:test-sigma-one-mode2} shows the PINN loss landscape: if the initial variance in $p_\theta$ is not close to zero  (the critical value is approximately $1.5$),  $\theta_\sigma$ tends to infinity instead of zero.  So the basin of attraction is only for $\theta_{init}<1.5$.
   
For the WGS loss, we  test difference choices $\alpha$ and $\kappa^2$ in the test function $\varphi=\mathcal{N}(\alpha,\kappa^2)$. 
For all these different values, 
we see that the basin of attraction for the true solution $\theta^*_\sigma=0$
 on the  WGS loss landscape becomes larger than that of the PINN loss;
 particularly, when $\kappa$ is large enough, any  initial guess can converge in this example.
Therefore, the flexibility of test function in the WGS  can improve the loss landscape for more robust training if  $\kappa$ is large enough.

\subsection{ \texorpdfstring{$p^*$}{} is bi-modal (\texorpdfstring{$w^*=0.5$}{})}

\begin{figure}[htbp]
    \centering
\includegraphics[width=.75\textwidth]{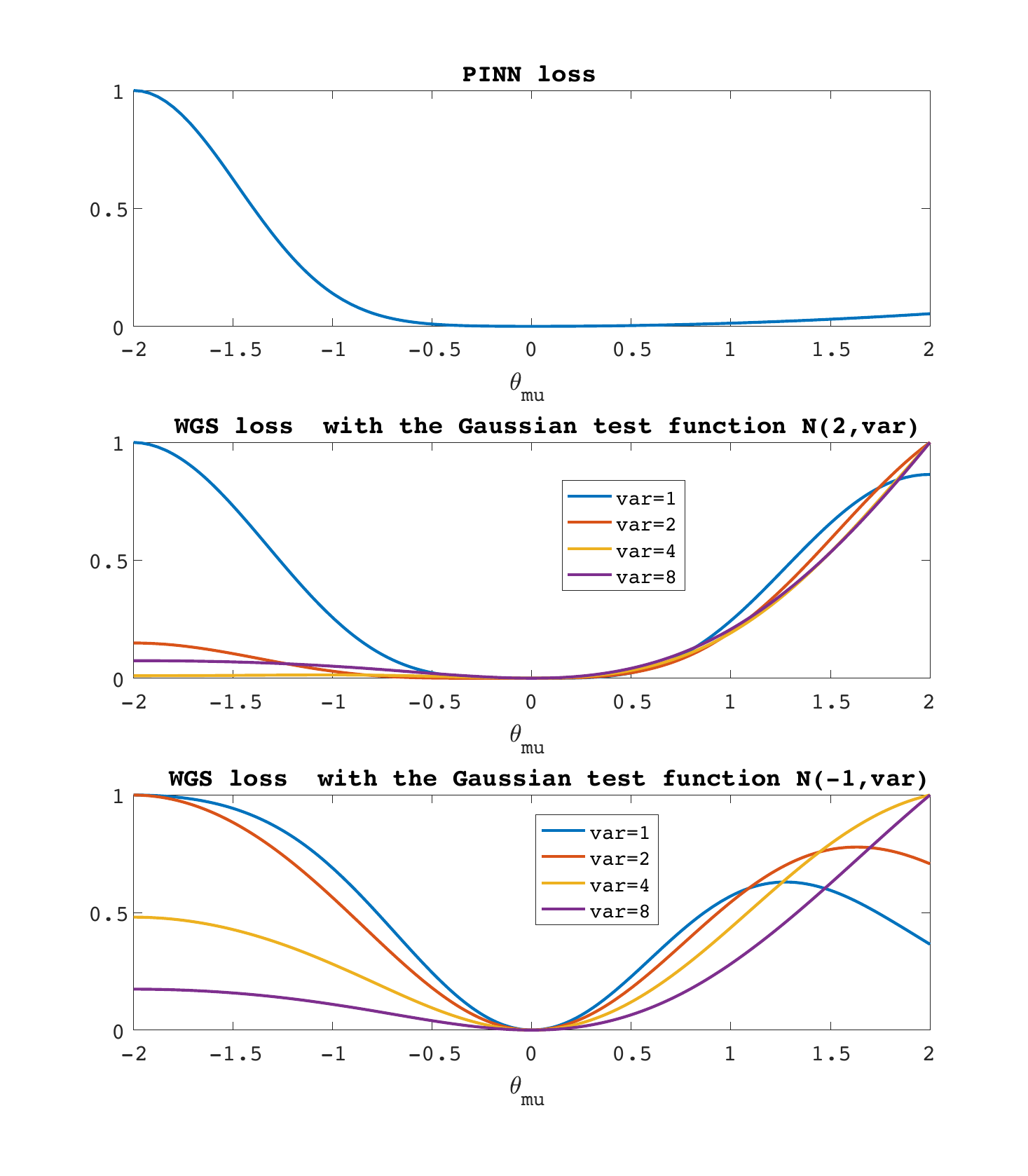}
    \caption{The loss functions of the PINN and the WGS associated vs. the parameter $\theta_\mu$,
    for the test function with the two centers located as $2$ (the middle panel) and $-1$ (the bottom panel)  respectively, and the four increasing values of $\kappa^2=1,2,4,8$   (``var" in the legend).  }
    \label{fig:test_mu}
\end{figure}

\begin{figure}[htbp]
    \centering
\includegraphics[width=.8\textwidth]{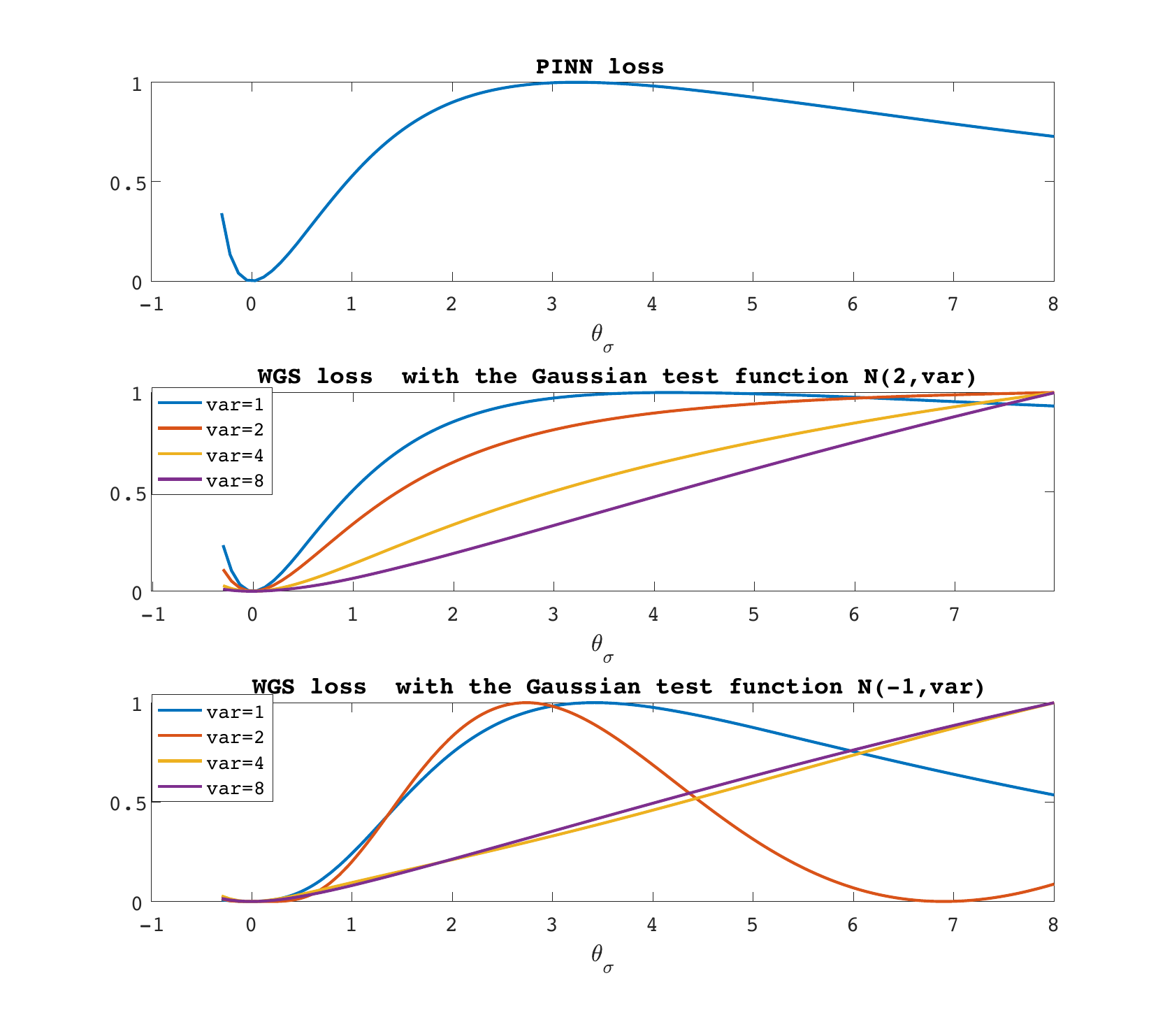}
    \caption{The loss functions of the PINN and the WGS associated vs. the parameter $\theta_\sigma$,
    for the test function with the two centers located as $2$ (the middle panel) and $-1$ (the bottom panel)  respectively, and the four increasing values of $\kappa^2=1,2,4,8$   (``var" in the legend).  }
    \label{fig:test_var}
\end{figure}

We further validate the above observation for the bi-modal case. 
We plot the losses in terms of $\theta_\mu$ in Figure \ref{fig:test_mu} (fixing $\theta_\sigma=\theta_w=0$) and in terms of    $\theta_\sigma$  in Figure \ref{fig:test_var} (fixing $\theta_w=\theta_\mu=0$).
We first discuss the effect of $\theta_\mu$ shown in Figure \ref{fig:test_mu}.   

The PINN loss landscape shown  in first subplot is almost flat when $\theta_\mu>0$,    
suggesting  the challenge for gradient descent method to optimize toward the truth $\theta_\sigma^*=0$ for any positive initial guess.  For the WGS losses in the second and third subplots, we observe
that the landscape is significantly improved, and the gradient flow can find the true solution at zero from a large domain of initial guesses.

Figure \ref{fig:test_var} is very similar to  Figure \ref{fig:test-sigma-one-mode2} for the single mode case, confirming the benefit of the WGS loss   in this bi-modal case again for optimizing $\theta_\sigma$.

 \smallskip
 The practical algorithm  uses a set of Gaussian test functions   with different $\mu$ and $\kappa$,  and the normalizing flow also parametrizes    $p_\theta$ in a  much more complicate way, thus the overall effect on the WGS loss landscape can be extremely difficult  to probe.  
But  by studying  the above 
simple example, we show the advantage of the WGS loss landscape 
than the PINN loss landscape when learning three parameters in $\theta$ by using even just one  test function in the WGS,  
particularly for a large hyper-parameter $\kappa$ in the test function. 
 Our   numerical experience for  real examples in this paper  suggests  that a large scale parameter in general helps stabilize the training but has difficulty to further improve the accuracy. So we recommended  a general adaptivity   strategy which proves effective: it is recommended to tune down the scale hyper-parameters $\kappa$ in test functions during the training steps or to use a several groups of   $\kappa$ with different magnitudes, in order to balance the robustness and the accuracy.

\section{Additional numerical comparison for the bi-modal Example 2
}\label{appendix:c}
\subsection{Plot of the ``potential'' function \texorpdfstring{$-\epsilon \log p$}{} in Example 2}\label{appendix:c1}
By defining $\hat{V}(x, y) := -\epsilon \log p(x, y)$, we can also compare the computational results of the invariant measures by plotting $\hat{V}$. In Figure \ref{fig:e2_potential}, we present this  potential $\hat{V}(x, y)$.  Since the most significant part of the potential lies in the neighboring regions of the two metastable states, we apply a truncation to the potential in these regions.  Specifically, we truncate the potential up to $0.8$ for $\epsilon = 0.2$,    up to $0.7$ for $\epsilon = 0.1$, and up to $0.5$ for $\epsilon = 0.05$.

\begin{figure}[htbp]
    \centering
    \includegraphics[scale=0.80]{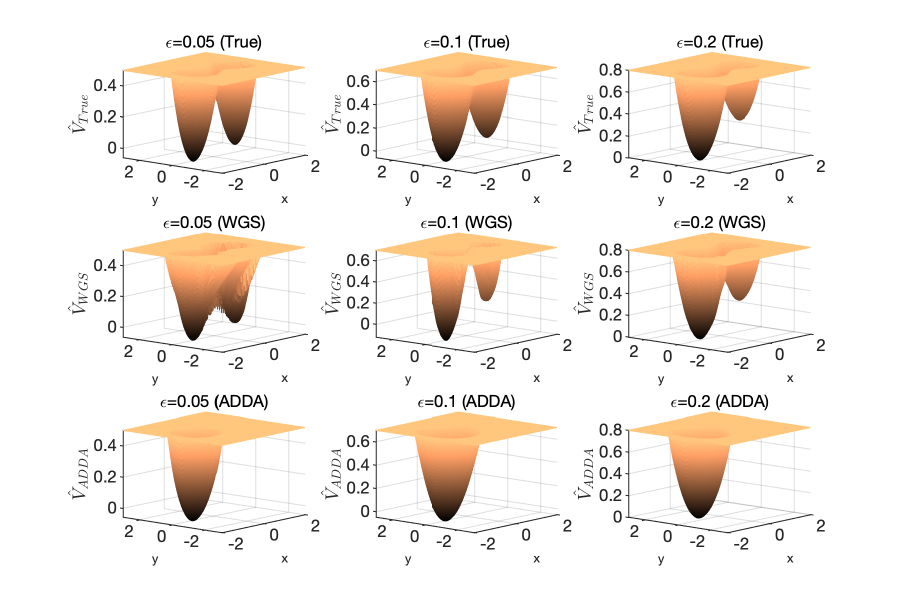}
    \caption{(Example 2) 3D mesh plots of the truncated ``potential''   $V=-\epsilon \log p $: $\hat{V}_{\text{True}}(x,y)$(top), $\hat{V}_{\text{WGS}}(x,y)$ learned by WGS (middle) and $\hat{V}_{\text{ADDA}}(x)$ learned by ADDA for $\varepsilon =0.05$ (left), $\varepsilon =0.1$ (middle) and $\varepsilon =0.2$ (right) }
    \label{fig:e2_potential}
\end{figure}

\subsection{Comparison of ADDA and WGS for Example 2}\label{appendix:c2}
We  show more results in detail about the performance in capturing the bi-modal distribution in Example 2.

The true  solution $p$ is computed by solving the stationary Fokker-Planck equation with the finite difference method on the domain $\Omega = [-2.5, 2.5] \times [-3, 3]$ with a uniform $400 \times 400$ grid.

\begin{algorithm}
\caption{Adaptive deep density approximation \cite{tang2022adaptive}}\label{alg:adda}
\SetKwInOut{Input}{Input}\SetKwInOut{Output}{Output}
\Input{Initial generative map $G_\theta$ and the base distribution $\rho$ with $p_\theta=G_{\theta\#}\rho$; training iteration number $N_I^p$; adaptive iteration number $N_{\text{adaptive}}$; initial distribution for training dataset $p_0$; the hyper-parameters $\lambda>0$ and $r>0$, $c>0$ for $L_b$. }

\For{$k=1:N_\text{adaptive}$}{
\uIf{$k=1$}{Generate an initial training dataset $\mathcal{D}=\{x_i\}_{i=1}^{N_p}$ from $p_0$;}

\Else{Sample $\{z_i\}_{i=1}^{N_p}$ from $\rho$\;
Obtain training dataset $\mathcal{D}=\{x_i\}_{i=1}^{N_p}$ by $x_i=G_\theta(z_i)$;
}
Split $\mathcal{D}$ into minibatches of size $N_p^b$ as \;
\For{$n=1:N_I^p$}{
\For{$m=1:\lceil N_p/N_p^b \rceil$}{Compute the Loss function:
\begin{equation*}
    L_{\text{ADDA}}=\frac{1}{N_p^b}\sum_{j=1}^{N_p^b}|\mathcal{L}p_\theta(x^m_j)|^2+\lambda L_b,
\end{equation*}\\
where $\{x_j^m\}_{j=1}^{N_p^b}$ is the $m$-th mini-batch dataset\;
Update the parameters $\theta$ using the Adam optimizer with a learning rate $\eta$\;
}

}
}
\Output{The trained transport map $G_\theta$}
\end{algorithm}

The ADDA method is specified in Algorithm \ref{alg:adda}. The training setting for the  ADDA  method in Example 2 is as follows. The initial training set for ADDA is generated from a uniform distribution over the range $[-4, 4]^2$ for all cases. The training dataset size is set to $N_p = 60,000$, with a batch size of $N_p^b = 2,000$. The number of training iterations is fixed at $N_I^p = 500$, and the number of adaptivity iterations is set to $N_{\text{adaptive}} = 5$ for each case. We use the same network structure to parameterize the generative map $G_\theta$ and the distribution $p_\theta=G_{\theta\#}\rho$ in WGS and ADDA. Each case is repeated for six independent runs.

The two modes in this example lie around the locations $(\pm1,0)$, we simply 
propose to check the following quantity 
\begin{equation}
    \label{305}\mbox{Prob}(X>0)=\int_{x>0}\int_{\mathbb{R}}p(x,y)\d x\d y
\end{equation}
to quantify if the distribution $p(x,y)$ captures the two modes very well.
For the true invariant measure, the true value of \eqref{305} is a number strictly between zero and one (marked by  the thick dashed horizontal line in Figure \ref{fig:e2_prob_x}).
If this probability in \eqref{305} is close to zero (or one), then the distribution $p$ has missed the mode on the right (or left, respectively).

\begin{figure}[htbp]
    \centering
    \includegraphics[scale=0.35]{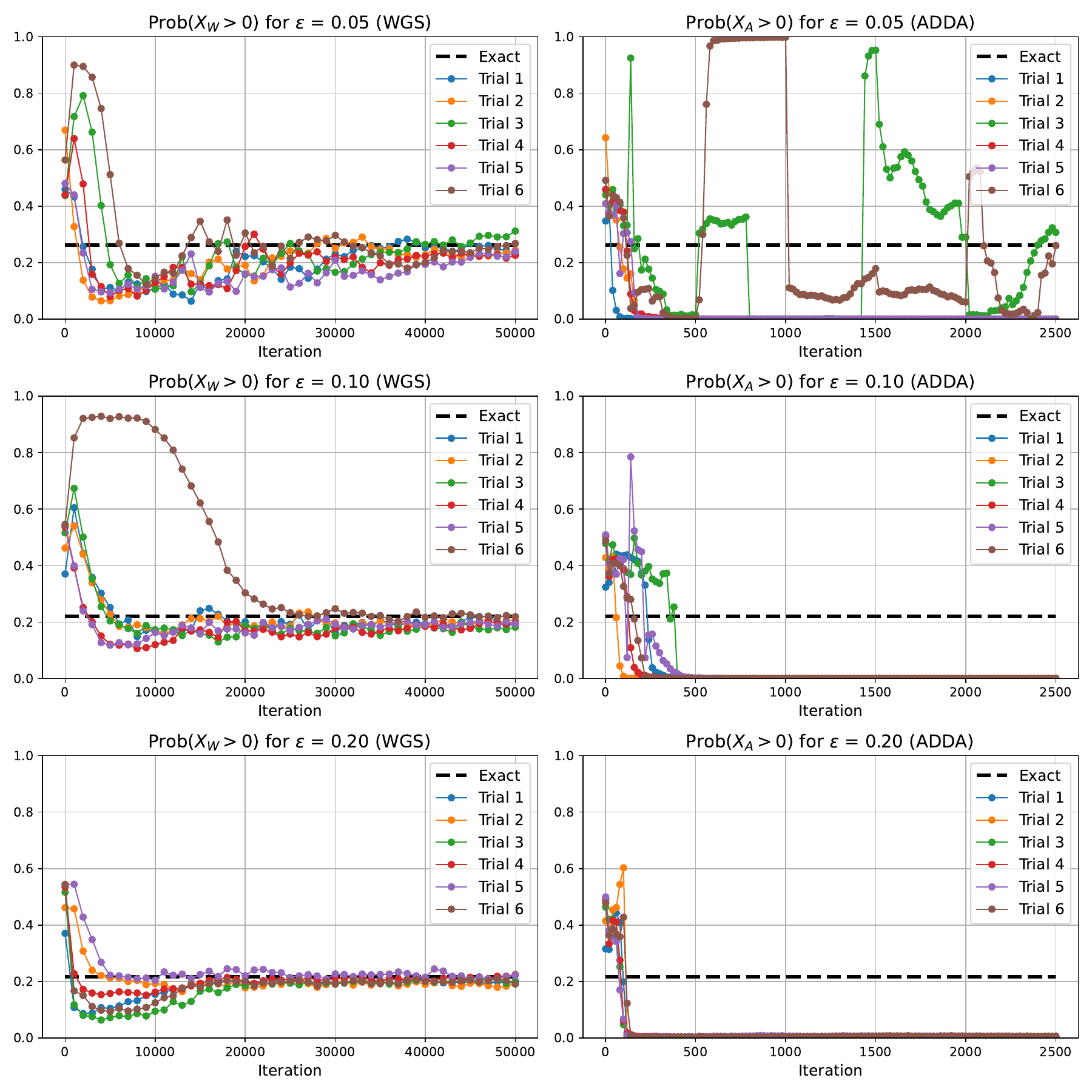}
    \caption{(Example 2) $\mbox{Prob}(X>0)$ versus the training epochs for different $\varepsilon$, compared between the WGS and ADDA methods across six independent trials. The black dotted line represents the true value of $\mbox{Prob}(X>0)$. Each curve corresponds to a trial with a different random seed. Curves approaching zero or one indicate failed trials with mode collapse.} 
    \label{fig:e2_prob_x}
\end{figure}
 Figure \ref{fig:e2_prob_x}   illustrates the evolutions of $\mbox{Prob}(X_W > 0)$ computed by WGS, where $p(x, y)$ is replaced with $p^W_\theta(x, y)$, and $\mbox{Prob}(X_A > 0)$ computed by ADDA, where $p(x, y)$ is replaced with $p^A_\theta(x, y)$, during the iterative process. Each test is repeated for six trials and each trial in the figure corresponds to a different random seed. 

We observe that WGS demonstrates greater robustness compared to ADDA based on the quantity  \eqref{305}. All trials of WGS converge to the exact value of $\mbox{Prob}(X > 0)$ for different values of $\varepsilon$. Notably, in trial $6$ with $\varepsilon = 0.1$, WGS initially converges to a single metastable state but eventually transitions to split across both metastable states. No trials of ADDA   converge to the exact value of $\mbox{Prob}(X>0)$ for all three $\varepsilon$ tested here. The numerical solutions of   ADDA   become trapped into the mode on the left only which has a higher probability  than the mode on the right, resulting in $\mbox{Prob}(X_A > 0)$ being zero. For the smallest $\varepsilon = 0.05$ tested here, the training of ADDA shows instability since the value of $\mbox{Prob}(X_A > 0)$ jumps back and forth between zero and one during the iterations, indicating the oscillation between two uni-modal distribution.
This also explains the large variance in the ADDA error in Table 1.

\section{Training details for Example 5}\label{appendix:d}
\begin{figure}[ht]
    \centering
    \includegraphics[width=0.9\linewidth]{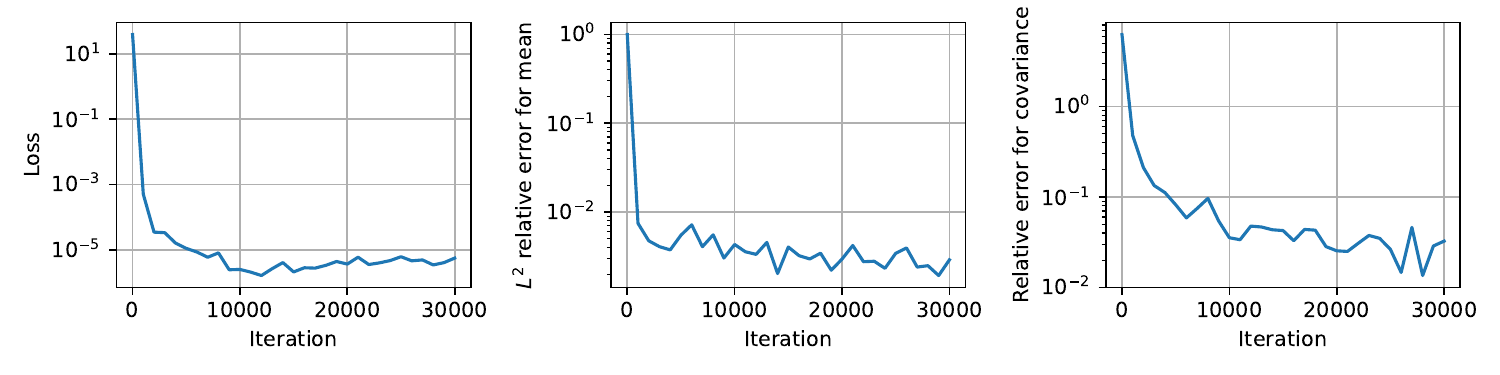}
      \includegraphics[width=0.9\linewidth]{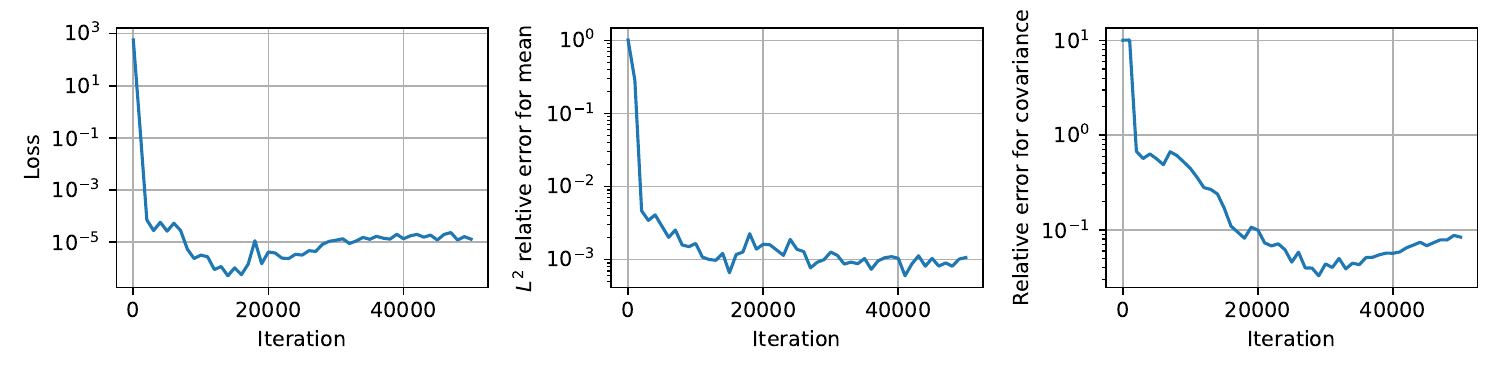}
    \caption{(Example 5) Training behavior in Example 5. The left panel illustrates the decay of the loss function with respect to iterations, while the middle and right panels depict the relative errors for the mean and covariance, respectively, versus the iteration. Top: 40 dimensional problem; Bottom: 100 dimensional problem}
    \label{fig:40100dim_loss}
\end{figure}

For $d=40$ case, we incorporate $20$ affine coupling layers into the Real NVP architecture, with each layer comprising a three-layer neural network with a width of $128$ units. The training process uses $N_I = 30,000$ iterations, with a dataset that comprises $N = 10,000$ sample points. We select $N_\varphi = 200$ test functions with the full batch size at each iteration. 

Throughout the training phase, we use two groups of $\kappa$ in  test functions. The first group is characterized by a fixed  $\kappa=11$, while the second type starts with the same value  $11$ but follows an exponential decay schedule. $\gamma$ was initially set at $0.7$ and then gradually decreased to $0.21$ during the training process. Furthermore, the learning rate was set to $10^{-4}$ and followed an exponential decay schedule.

For $d=100$ case, we use the same Real NVP   network,     $N_I = 50,000$ iterations,  $N = 10,000$ sample points and $N_\varphi = 300$ test functions. Throughout the training phase, we use three groups of test functions (each group has $100$ test functions). The first group has the common fixed scale hyper-parameter  of $\kappa=15$, while the second group all starts with the same   $\kappa=15$  but then follows an exponential decay schedule. Each $\kappa$ of the third group is sample randomly  from the uniform distribution between $7$ and $9.8$.  $\gamma$ was initially set at $0.7$ and then gradually decreased to $0.14$ during the training process. Furthermore, the learning rate was set to $10^{-4}$ and followed an exponential decay schedule.

Figure \ref{fig:40100dim_loss} below
plots the loss, the relative errors (of the mean and the variance) for $d=40$ and $d=100$ to show the training behavior.

\section{Discussion of the hyper-parameters}\label{appendix:e}

\begin{table*}[htbp]
\footnotesize
\caption{The choice of 
$\gamma$ and the hyper-parameters of the boundary loss for examples in the main text }
\label{tab_hp}
\begin{threeparttable}  
\begin{tabularx}{\textwidth}{XXXXXX}
\toprule
Example & $\gamma$  & $\lambda$ & $c$ & $x_0$& $r$\\
\midrule
$1$ & $0.5$ & $10$ & $6$ &$(0,0)$& $6$\\
\midrule
$2(\varepsilon =0.2)$ & $0.8$ & $20$ & $10$ &$(0,0)$& $4$\\
$2(\varepsilon =0.1)$ & $0.8$ & $20$ & $10$ &$(0,0)$& $4$\\
$2(\varepsilon =0.05)$ & $0.8$ & $20$ & $10$&$(0,0)$& $4$\\
\midrule
$3$ & $5$ & $5$ & $5$ &$(0,0,25)$& $30\times40\times40$\tnote{1}\\
\midrule
$4$ & $0.3$ & $5$ & $6$ &$(0,0,\cdots,0)$& $2$\\
\midrule
$5$& $0.7\downarrow$\tnote{2} & $10$ & $6$ &$(0,0,\cdots,0)$& $6$\\
\bottomrule
\end{tabularx}
\begin{tablenotes}
        \footnotesize               %这行要添加
        \item[1] Each number represents the radius in each coordinate.  
        \item[2] $\downarrow$ represents   an exponential decay schedule. 
\end{tablenotes}
\end{threeparttable}
\end{table*}

In Table \ref{tab_hp}, we list the hyper-parameters utilized in our numerical experiments. The first hyper-parameter, $\gamma$, governs the noise level in the mean of the test function,  and the other hyper-parameter in this table influences the boundary penalty, as described in the definition of the boundary loss 
$
L_b=\frac{1}{N}\sum_{i=1}^N \text{Sigmoid}\Big(c\big(\|G(z_i)-x_0\|_2^2-r^2\big)\Big).
$ We use this boundary loss to ensure that nearly all sample data points remain within the ball centered at $x_0$ with radius $r$.

To further investigate the impact of hyper-parameter selections on the outcomes, we conduct a series of sensitivity analysis experiments for the example in Section 5 with $d=40$. 
We keep all hyper-parameters  the same as in Appendix \ref{appendix:d} except the one under the tuning test. All   experiments are repeated  over five independent runs. To estimate the accuracy, we compute the weak loss and the relative errors for the mean and covariance matrix, as done in Example 5. We present only the average  of these outcomes from these independent runs.

\begin{figure}[ht]
    \centering
    \includegraphics[width=0.75\linewidth]{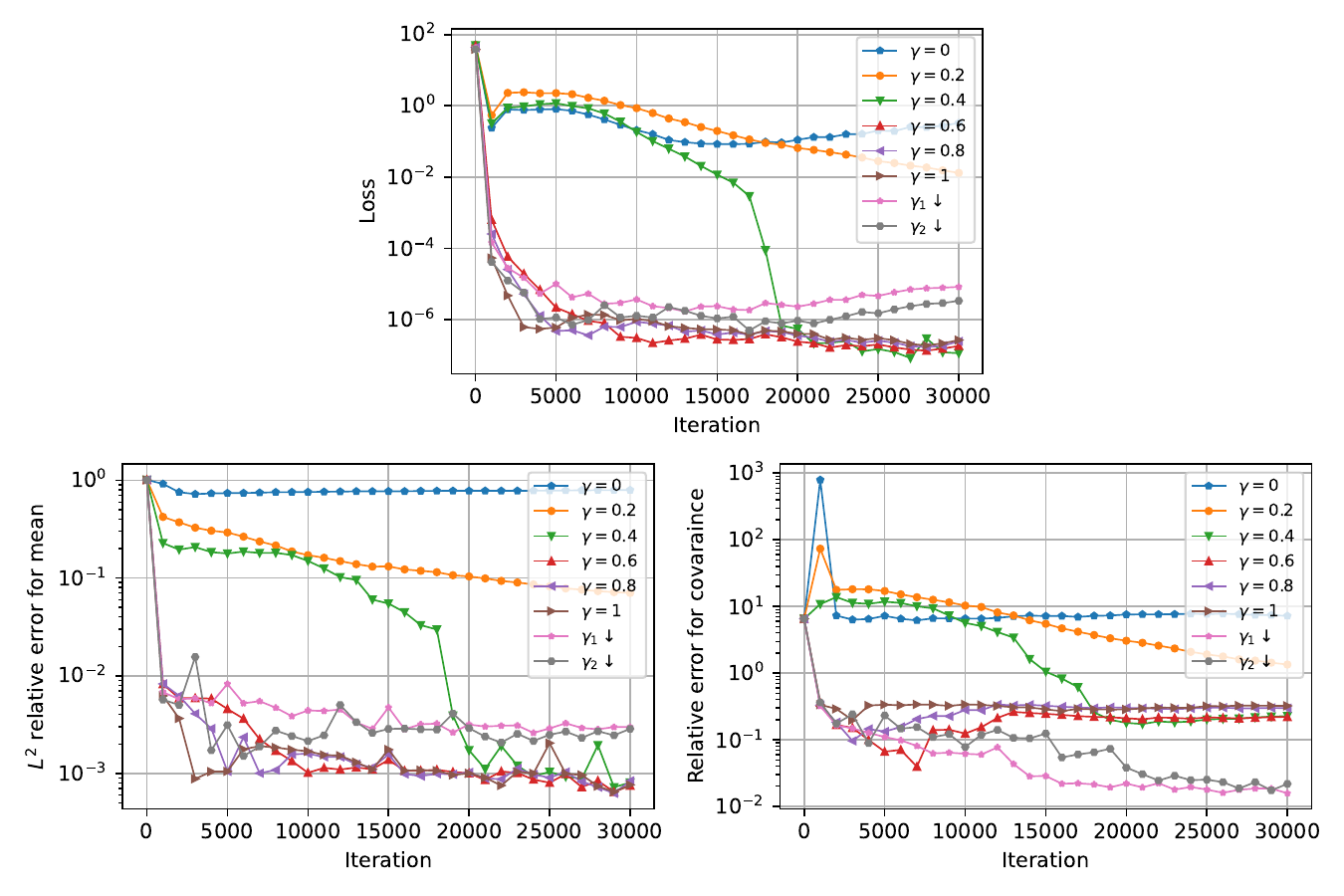}
    \caption{Comparison of different choices of $\gamma$. The upper panel illustrates the WGS loss versus the iterations, while the lower left panel displays the relative error for the mean (left) and the covariance matrix (right). Here, $\gamma_1\downarrow$ indicates that $\gamma$ gradually decreases from $0.8$ to $0.24$, whereas $\gamma_2\downarrow$ indicates that $\gamma$ gradually decreases from $1$ to $0.3$.}
    \label{fig:discussion_params_gamma}
\end{figure}

The hyper-parameter $\gamma$ regulates the additional noise   in the centers of the test function.   As shown in Figure \ref{fig:discussion_params_gamma},  the zero or tiny $\gamma$ value shows a slow convergence and thus it is 
beneficial to allow a relatively large $\gamma$.  Additionally, a decay schedule is recommended for selecting $\gamma$ to   enhance the training process as we done in Example 5. For multi-modal problems, it might be particularly important to tune $\gamma$ in this way for better exploration
to mitigate the issue of mode collapse.

The hyper-parameters   $r$ , $c$ and  $\lambda$   for the boundary loss function are also tested.
As shown in Figure \ref{fig:discussion_params}, different choices of $\lambda$, $c$, $r$, do not influence the WGS loss or the relative error for the mean and covariance matrix in our method.  
\begin{figure}[htbp]
    \centering
    \includegraphics[width=0.9\linewidth]{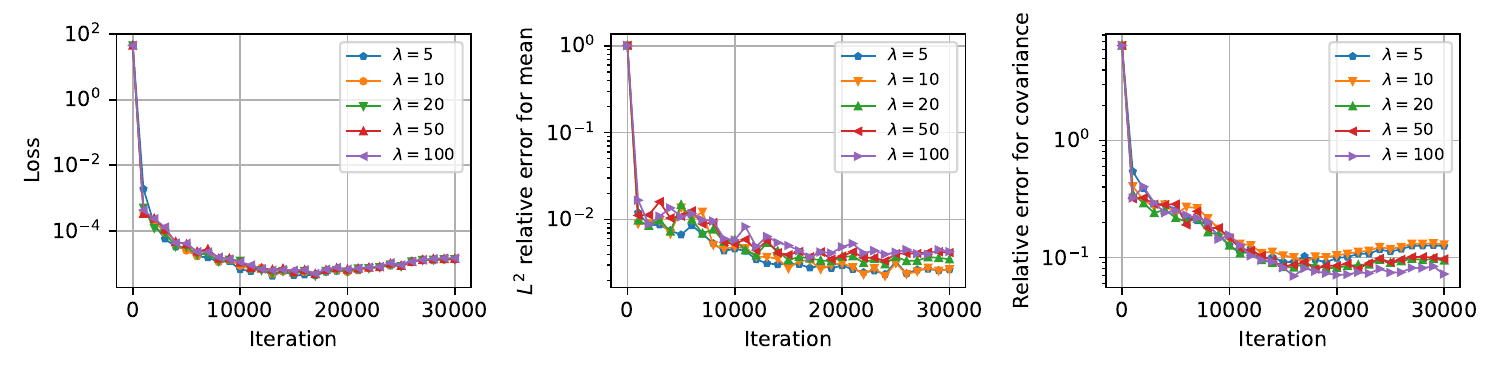}
    \includegraphics[width=0.9\linewidth]{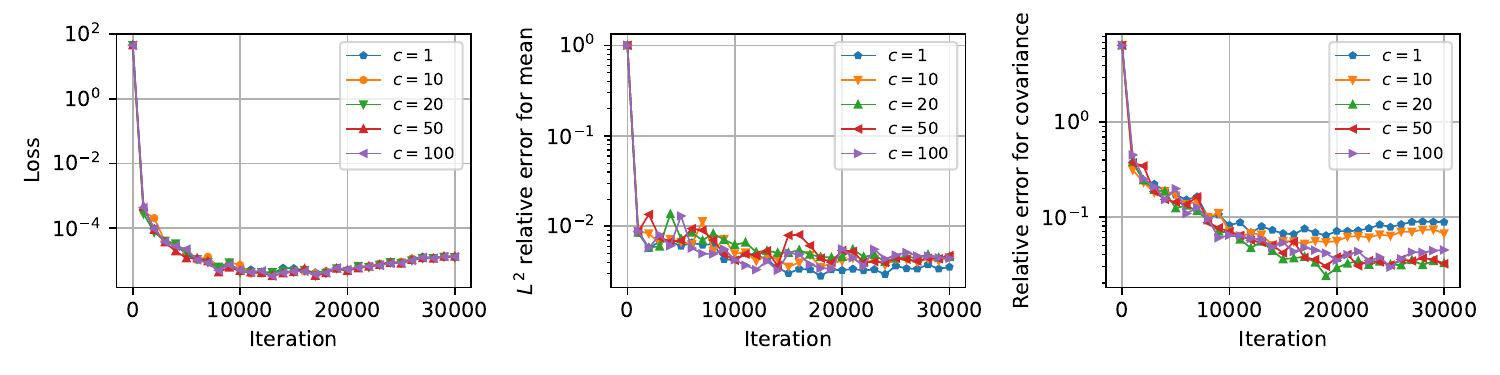}
    \includegraphics[width=0.9\linewidth]{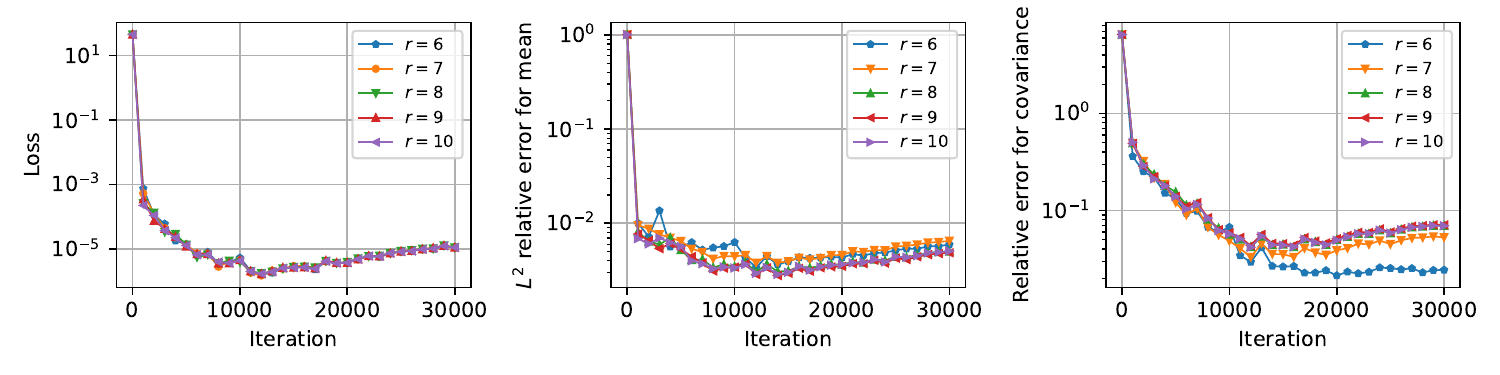}
    
    \caption{Comparison of different choices of $\lambda$ (top), $c$ (middle) and $r$ (bottom) for the boundary loss. The right panel shows the WGS loss versus the iterations, while the middle panel and the left panel present the relative error for the mean and the covariance matrix.}
    \label{fig:discussion_params}
\end{figure}

% \nocite{*}
\bibliography{ref}
 
\bibliographystyle{plain} %alpha, siam

\end{document}